\documentclass{article}

\PassOptionsToPackage{numbers, compress}{natbib}
\usepackage[final]{neurips_2023}

\usepackage[utf8]{inputenc} 
\usepackage[T1]{fontenc}    
\usepackage{hyperref}       
\usepackage{url}            
\usepackage{booktabs}       
\usepackage{amsfonts}       
\usepackage{nicefrac}       
\usepackage{microtype}      
\usepackage{xcolor}
\usepackage{wrapfig}
\usepackage{booktabs} 
\usepackage{natbib}
\usepackage{float}
\PassOptionsToPackage{options}{natbib}

\usepackage{amsthm,amsmath,amssymb}

\usepackage{graphicx}
\usepackage{subcaption}
\usepackage{caption}
\usepackage{dsfont}

\newtheorem{theorem}{Theorem}[section]
\newtheorem{proposition}[theorem]{Proposition}
\newtheorem{example}[theorem]{Example}
\newtheorem{lemma}[theorem]{Lemma}
\newtheorem{corollary}[theorem]{Corollary}

\def\bfw{\mathbf{w}}

\newcommand*\samethanks[1][\value{footnote}]{\footnotemark[#1]}

\definecolor{wjs}{RGB}{0,0,200}

\title{Unified Enhancement of Privacy Bounds for Mixture Mechanisms via $f$-Differential Privacy}

\author{%
  Chendi Wang\thanks{Equal contributions.}\\
  Wharton Statistics and Data Science Department\\
  University of Pennsylvania \&\\
  Shenzhen Research Institute of Big data \\
\texttt{chendi@wharton.upenn.edu} \\
  \And
  Buxin Su\samethanks \\
  Department of Mathematics\\
  University of Pennsylvania\\
  \texttt{subuxin@sas.upenn.edu} \\
    \AND
  Jiayuan Ye\\
  Department of Computer Science\\
  National University of Singapore \\
  Singapore \\
  \texttt{jiayuan@comp.nus.edu.sg} \\
  \And
  Reza Shokri\\
  Department of Computer Science\\
  National University of Singapore \\
  Singapore \\
  \texttt{reza@comp.nus.edu.sg} \\
  \And
  Weijie J.~Su \\
  Wharton Statistics and Data Science Department\\
  University of Pennsylvania\\
  \texttt{suw@wharton.upenn.edu} \\
}

\begin{document}

\maketitle

\begin{abstract}


Differentially private (DP) machine learning algorithms incur many sources of randomness, such as random initialization, random batch subsampling, and shuffling. However, such randomness is difficult to take into account when proving differential privacy bounds because it induces mixture distributions for the algorithm's output that are difficult to analyze. 
This paper focuses on improving privacy bounds for shuffling models and one-iteration  differentially private gradient descent (DP-GD) with random initializations using $f$-DP. 
We derive a closed-form expression of the trade-off function for shuffling models that outperforms the most up-to-date results based on $(\epsilon,\delta)$-DP.
Moreover, we investigate the effects of random initialization on the privacy of one-iteration DP-GD. 
Our numerical computations of the trade-off function indicate that random initialization can enhance the privacy of DP-GD.
Our analysis of $f$-DP guarantees for these mixture mechanisms relies on an inequality for trade-off functions introduced in this paper. This inequality implies the joint convexity of $F$-divergences. 
Finally, we study an $f$-DP analog of the advanced joint convexity of the hockey-stick divergence related to $(\epsilon,\delta)$-DP  and apply it to analyze the privacy of mixture mechanisms.


\end{abstract}

\section{Introduction}
Differential privacy (DP, \cite{dwork2006differential,DBLP:conf/tamc/Dwork08}) is a rigorous mathematical framework for ensuring data privacy and has become a cornerstone of privacy-preserving data analysis over the past two decades. DP has found widespread applications in various data science fields, such as machine learning \cite{DBLP:journals/jmlr/ChaudhuriMS11, DBLP:conf/focs/BassilyST14, xu2023binary}, query answering \cite{DBLP:conf/stoc/DworkNRRV09, DBLP:conf/nips/DongSZ21}, and synthetic data generation \cite{DBLP:journals/corr/abs-2001-09700,DBLP:journals/tods/ZhangCPSX17, DBLP:journals/jpc/McKennaMS21,li2023statistical}.
A randomized mechanism is considered differentially private if the outputs of two neighboring datasets that differ in at most one element are indistinguishable from each other. 
The closeness of these outputs can be measured in various ways, resulting in the definition of $(\epsilon,\delta)$-DP in \cite{dwork2006differential} and its various relaxations.


The distinguishability between the outputs can be measured by statistical divergences.
For example, $(\epsilon,\delta)$-DP is associated with the so-called hockey-stick divergence \cite{DBLP:journals/tit/SasonV16}.
Another divergence relevant to differential privacy is the R\'enyi divergence \cite{DBLP:journals/corr/DworkR16,DBLP:journals/corr/abs-1206-2459} which leads to R\'enyi DP \cite{DBLP:conf/csfw/Mironov17, DBLP:conf/stoc/BunDRS18} and concentrated DP \cite{DBLP:journals/iacr/BunS16}.
In addition to divergence-based DP, a hypothesis testing perspective on differential privacy was proposed in \cite{Wasserman2010statistical}. More recently, \cite{dong2019gaussian} established $f$-DP 
for differential privacy where the privacy is measured by the trade-off function of type I and type II errors.

In real-world applications of differential privacy, including differentially private machine learning, it is common to analyze the privacy budget of mechanisms that involve mixture distributions, where the mixture is introduced by stochastic components in the algorithm. Examples of such mechanisms include sub-sampled mechanisms \cite{balle2018privacy,Zhu2019Poisson,Wang2019subsampled,mironov2019r}, shuffled mechanisms \cite{cheu2019shuffling,feldman2022hiding,feldman2023stronger}, and variants of the differentially private stochastic gradient descent (DP-SGD) algorithm~\cite{abadi2016DPSGD,Bu2020Deep,Kairouz2021practical,DBLP:conf/nips/AndrewTMR21,ye2022differentially} that involves random initialization and multiple rounds of mini-batch sampling. 
Recently, privacy amplification by iteration \cite{DBLP:conf/focs/FeldmanMTT18} has drawn much attention as it can be used to analyze the privacy bounds for DP-SGD  \cite{ye2022differentially,DBLP:conf/nips/AltschulerT22}) which leads to tighter privacy bounds compared to classical analysis based on the composition theorem \cite{DBLP:journals/corr/OhV13,wang2022analytical,DBLP:conf/icml/ZhengDLS20}.


While mixture mechanisms are essential in differentially private machine learning, the absence of an $f$-DP guarantee for their analysis remains a significant challenge. Moreover, existing divergence-based DP bounds for most of these mechanisms are not tight. This is primarily because the complex distribution resulting from the mixture makes it challenging to accurately quantify privacy guarantees. In order to illustrate this perspective, we consider the examples of shuffling models and DP gradient descent (DP-GD) with random initialization, as follows.

\begin{itemize}
    \item In shuffling models, each user's data record is locally privatized using a local DP algorithm \cite{DBLP:conf/pods/EvfimievskiGS03}.
    Subsequently, a curator shuffles the dataset containing all users' data. The shuffling procedure introduces additional mixtures of binomial noise \cite{feldman2022hiding}, thereby potentially amplifying the privacy provided by the local randomizer. Shuffling is commonly employed in machine learning algorithms for batch generation \cite{ye2022differentially,DBLP:journals/corr/abs-2109-12298}.
    To deal with this mixture, Hoeffding’s inequality was used in previous literature \cite{feldman2022hiding,feldman2023stronger} that leads to the loss of information.
    Using $f$-DP in this paper, we derive an exact analytical trade-off function for the mixture of binomial distributions which is sharp.

    \item In deep learning, random initialization is usually adopted in the stochastic gradient descent to enhance the performance of deep neural networks \cite{DBLP:conf/icml/SutskeverMDH13}.
    Intuitively, the inherent randomness introduced by initialization should contribute to the privacy amplification of DP-GD. However, R\'enyi differential privacy (DP) falls short in quantitatively measuring this randomness, even when applied to the simplest linear model. In this paper, we demonstrate how $f$-DP can effectively evaluate and quantify this inherent randomness from initialization.
\end{itemize}



\textbf{Our contributions.}
This paper makes a two-fold contribution. 
Firstly, we propose a unified theory to analyze the privacy of mixture mechanisms within the framework of $f$-DP.
Precisely, we derive an $f$-DP inequality for mixture distributions which implies the joint convexity of $F$-divergences for any convex $F$. We name this result the "joint concavity of trade-off functions", as it is a lower bound for trade-off functions. 
The tightness of the joint concavity is also investigated.
Moreover, we propose the "advanced joint concavity of trade-off functions" which is an $f$-DP analog of the advanced joint convexity of the hockey-stick divergence and results in sharper bounds in certain cases.

Building on our inequality, we have refined the privacy analysis of both shuffling models and DP-GD with random initialization using
$f$-DP. Specifically, for shuffling models, we obtain trade-off functions in a closed-form representation, leading to tighter bounds compared to existing state-of-the-art results based on $(\epsilon,\delta)$-DP. 
As for DP-GD, given the challenges in the trajectory analysis of multi-step iterations, we have chosen to explore a more straightforward one-iteration DP-GD. We demonstrate that using random initialization significantly enhances the privacy of the output from a single iteration.




\section{Preliminaries on differential privacy}
Let $\mathcal{D} = \{z_i\}_{i=1}^n\subset \mathcal{Z}$ be a fixed dataset of size $n$.
Consider a randomized algorithm $\mathcal{A}:\mathcal{Z}^n \rightarrow \mathcal{S}$ that maps a dataset $\mathcal{D}$ to $\mathcal{A}(\mathcal{D}) $ in some probability space $\mathcal{S}$.
Differential privacy requires that the change of one element in a dataset has a restricted impact on the output of $\mathcal{A}.$
Mathematically, we say $\mathcal{A}$ satisfies $(\epsilon,\delta)$-DP for some $\epsilon\geq 0$ and $0\leq \delta \leq 1$ if 
\begin{align*}
    \mathbb{P}[\mathcal{A}(\mathcal{D}_0)\in S]\leq e^{\epsilon}\mathbb{P}[\mathcal{A}(\mathcal{D}_1)\in S]  + \delta,
\end{align*}
for any event $S\in\mathcal{S}$ and any neighboring datasets $\mathcal{D}_0$ and $\mathcal{D}_1$.
When $\delta=0$, we simply call $(\epsilon,0)$-DP as $\epsilon$-DP.
Based on the definition, we see that for small values of $\epsilon$ and $\delta$, it is challenging to distinguish between $\mathcal{D}_0$ and $\mathcal{D}_1$ based on the outputs of $\mathcal{A}(\mathcal{D}_0)$ and $\mathcal{A}(\mathcal{D}_1)$, as the distribution of $\mathcal{A}(\mathcal{D}_0)$ closely resembles that of $\mathcal{A}(\mathcal{D}_1)$.

The definition of $(\epsilon,\delta)$-DP corresponds to the hockey-stick divergence.
Let $P$ and $Q$ be two distributions with  probability density functions (pdfs) $p$ and $q$, respectively.
The hockey-stick divergence between $P$ and $Q$ is defined by
$
    H_{\gamma}(P\|Q) = \int \left( p(x) - \gamma q(x)\right)_{+} dx
$
for $\gamma\geq 1$ with $(\cdot)_{+} = \max\{0,\cdot\}$.
With a little bit abuse of notations, in this paper, we define the divergence (or the trade-off function) between two random variables as the divergence (or the trade-off function) between their distributions.
Then, a mechanism $\mathcal{A}$ is $(\epsilon,\delta)$-DP if and only if $H_{e^{\epsilon}}(\mathcal{A}(\mathcal{D}_0)\| \mathcal{A}(\mathcal{D}_1)) \leq \delta$ for any neighboring datasets $\mathcal{D}_0$ and $\mathcal{D}_1$, which also implies $H_{e^{\epsilon}}(\mathcal{A}(\mathcal{D}_1)\| \mathcal{A}(\mathcal{D}_0)) \leq \delta.$

The R\'enyi-DP (RDP) is defined based-on the R\'enyi divergence. The R\'enyi divergence of order $\widetilde{\alpha}>1$ between $P$ and $Q$ is given by
\begin{align*}
  R_{\widetilde{\alpha}}(P\|Q) = \frac{1}{\widetilde{\alpha} - 1} \log \int \left(\frac{p(x)}{q(x)} \right)^{\widetilde{\alpha}} q(x) dx.
\end{align*}
For ${\widetilde{\alpha}} =1$ or $+\infty$, $R_1$ or $R_{\infty}$ is the limit of $R_{\widetilde{\alpha}}$ as $\widetilde{\alpha}$ tends to $1$ or $+\infty$.
A mechanism $\mathcal{A}$ is said to satisfy $(\widetilde{\alpha},\epsilon)$-RDP if $R_{\widetilde{\alpha}}(\mathcal{A}(\mathcal{D}_0)\|\mathcal{A}(\mathcal{D}_1))\leq \epsilon$ for any neighboring $\mathcal{D}_0$ and $\mathcal{D}_1$.

The distinguishability between $\mathcal{A}(\mathcal{D}_0)$ and $\mathcal{A}(\mathcal{D}_1)$ can be quantified using hypothesis testing, which aligns with the concept of $f$-DP.
Consider a hypothesis testing problem $H_0: P \hbox{ v.s. } H_1:Q$ and a rejection rule $\phi\in[0,1].$ We define the type I error as $\alpha_{\phi} = \mathbb{E}_{P}[\phi]$, which is the probability that we reject the null hypothesis $H_0$ by mistake. 
The type II error $\beta_{\phi}= 1 - \mathbb{E}_{Q}[\phi]$ is the probability that we accept the alternative $H_1$ wrongly.

The trade-off function $T(P,Q)$ is the minimal type II error at level $\alpha$ of the type I error, that is,
\begin{align*}
    T(P,Q)(\alpha) = \inf_{\phi}\{\beta_{\phi}: \alpha_{\phi}\leq \alpha\}.
\end{align*}
We say a mechanism $\mathcal{A}$ satisfies $f$-DP if $
    T(\mathcal{A}(\mathcal{D}_0), \mathcal{A}(\mathcal{D}_1)) \geq f
$
for any neighboring datasets $\mathcal{D}_0$ and $\mathcal{D}_1$. In particular, $\mathcal{A}$ is said to satisfy $\mu$-GDP if it is $G_{\mu}$-DP, where $G_\mu(x) = \Phi(\Phi^{-1}(1-x) - \mu)$, for $\mu\geq 0$, is the Gaussian trade-off function with $\Phi$ being the cumulative distribution function (cdf) of $\mathcal{N}(0,1).$
$\mathcal{A}$ is considered to be more private if the corresponding trade-off function takes larger values.
When $\mathcal{A}$ achieves perfect privacy and $\mathcal{A}(\mathcal{D}_0)$ and $\mathcal{A}(\mathcal{D}_1)$ become completely indistinguishable, the trade-off function is $\mathrm{Id}(x) = 1 - x$. Consequently, for any trade-off function $f$, we have $f \leq \mathrm{Id}$.

We say a trade-off function  is symmetric if $T(P,Q) = T(Q,P)$.
Note that a trade-off function $f$ may not necessarily be symmetric.
But one can symmetrize it as shown in \cite{dong2019gaussian}. The symmetrization of a trade-off function will be used when we analyze the shuffled mechanisms.

\section{Joint concavity of trade-off functions}
\label{sec:joint-convex}
Let $\{ P_i \}_{i=1}^m$ and $\{Q_i\}_{i=1}^m$ be two sequences of probability distributions.
Denote the probability density functions (pdfs) of $P_i$ and $Q_i$ as $p_i$ and $q_i$, respectively.
Consider the mixture distributions $P_\bfw$ and $Q_\bfw$ with pdfs $p_\bfw = \sum_{i=1}^m w_i p_i$ and $q_\bfw = \sum_{i=1}^m w_iq_i$,
where the weight $\bfw = (w_1,\cdots,w_m)$ is such that $w_i\geq 0$ and $\sum_{i=1}^m w_i = 1.$
The following lemma is to bound the trade-off function $T(P_\bfw,Q_\bfw).$
Upon finalizing this paper, we noted that Lemma \ref{lemma:inequality} and Proposition \ref{prop:symmetry} appeared independently in another paper \cite[Theorem 8]{wang2022differentially}, where they served different applications.

\begin{lemma}[Joint concavity of trade-off functions] \label{lemma:inequality}
For two mixture distributions $P_\bfw$ and $Q_\bfw$, it holds
\begin{align*}
    T(P_\bfw,Q_\bfw)(\alpha(t,c)) \geq \sum_{i=1}^m w_i T(P_i,Q_i)(\alpha_i(t,c)),
\end{align*}
where $\alpha_i(t,c) = \mathbb{P}_{X\sim P_i}\left[ \frac{q_i}{p_i}(X) > t\right] + c \mathbb{P}_{X\sim P_i}\left[ \frac{q_i}{p_i}(X) =t\right]$ is the type I error for testing $P_i \hbox{ v.s. } Q_i$ using the likelihood ratio test  and $\alpha (t,c) = \sum_{i=1}^m w_i \alpha_i(t,c).$
\end{lemma}

The main idea of the proof is to make the mixture distributions more distinguishable by releasing the indices. 
Precisely, for $X\sim P_\bfw$ and $Y\sim Q_\bfw$, let $X|I$ be a random variable such that $X|I=i\sim P_i$ with $I$ being the indices, i.e., $\mathbb{P}[I = i] = w_i$.
Let $(X|I,I)$ be a random variable where we observe both $X|I$ and the indices $I$.
Then, the right hand side of Lemma \ref{lemma:inequality} is the trade-off function $T((X|I,I), (Y|I,I))$ between two joint distributions. This is a lower bound for the trade-off function between mixture distributions because
$(X|I,I) \rightarrow X$ is a data-independent post-processing procedure that only removes the observation of indices $I$, and DP is immune to post-processing \cite{dong2019gaussian, DBLP:journals/fttcs/DworkR14}.


Under the setting of $f$-DP, we usually require that the trade-off function is symmetric.
The symmetry of the trade-off function in Lemma \ref{lemma:inequality} is guaranteed by the following proposition.
\begin{proposition}
\label{prop:symmetry}
    Suppose that for each $i$, $T(P_i,Q_i)$ is a symmetric trade-off function. Then the trade-off function $T((X|I, I) , (Y|I,I))$ is symmetric.
\end{proposition}

The joint convexity of $F$-divergences plays an important role in the analysis of divergence-based DP for mixture mechanisms \cite{balle2018privacy, feldman2022hiding}.
We now show that Lemma \ref{lemma:inequality} is an extension of the joint convexity of $F$-divergences, including the scaled exponentiation of the R\'enyi divergence and the hockey-stick divergence, to trade-off functions.
A trade-off function is always convex and is thus differentiable almost everywhere.
Thus, without loss of generality, we consider $f_i$  that is differentiable, symmetric, with $f_i(0)=1$.

\begin{proposition}
[An application of Lemma \ref{lemma:inequality} to the $F$-divergences]
\label{prop:F-divergence}
Let $D_F(P\|Q) = \int F(p(x)/q(x))q(x)dx$ be an $F$-divergence between any two distributions $P$ and $Q$ with some convex $F$.
Then, for $f_i = T(P_i,Q_i)$, we have
\begin{align*}
    D_F(P_\bfw\|Q_\bfw) \leq \sum_{i=1}^m \int_{0}^1 F\left(\frac{1}{|f_i'(x)|} \right) \left|f_i'(x) \right|dx = \sum_{i=1}^m w_i D_F(P_i\|Q_i).
\end{align*}
\end{proposition}
Conversion from a trade-off function to $F$-divergences is straightforward using Section B in \cite{dong2019gaussian}. However, conversion from an $F$-divergence to a trade-off function is highly non-trivial. In fact, $F$-divergence is an integral of a functional of the trade-off function over the whole space while Lemma \ref{lemma:inequality} holds pointwisely, which is a local property. 
This explains why the divergence-based DP is not as informative as $f$-DP since some information is lost due to the integration.

\section{Privacy analysis of the shuffled mechanisms}
In this section, we explore the 
$f$-DP analysis of shuffled mechanisms. Drawing upon \cite{feldman2022hiding, feldman2023stronger}, the shuffling procedure incorporates a mixture of binomial noise. This noise can be tightly bounded using our 
$f$-DP inequality for mixture distributions.

\subsection{Theoretical privacy guarantee}
\label{sec:shuffling}
In shuffling models, the record of each user is privatized by some local randomizer (such as a randomized response mechanism \cite{DBLP:conf/edbt/0009WH16})  and all records are then shuffled by a curator.
Mathematically, consider a dataset $\mathcal{D} = \{z_i\}_{i=1}^n\subseteq\mathcal{Z}$ of size $n$ and each data point $z_i$ is privatized by an local randomizer $\mathcal{A}_0:\mathcal{Z}:\rightarrow\widetilde{\mathcal{Z}}$ that satisfies $\epsilon_0$-DP.
Then, the mechanism $\mathcal{A}:\mathcal{Z}^n \rightarrow \widetilde{\mathcal{Z}}^n$ that maps $\mathcal{D}$ to $\widetilde{\mathcal{D}}=\{\mathcal{A}_0(z_i)\}_{i=1}^n$ is $\epsilon_0$-DP.
A shuffler $\mathcal{A}_{\mathrm{Shuffle}}$ takes the privatized dataset $\widetilde{\mathcal{D}}$ as input and applies a uniformly random
permutation to $\widetilde{\mathcal{D}}$, which introduces the mixture of binomial noise to $\mathcal{A}$ and results in privacy  amplification.

As noted in \cite{feldman2023stronger}, the shuffling procedure introduces mixtures of binomial distributions.
More specifically,  the outputs generated by the shuffled mechanism for two neighboring datasets result from post-processing random variables  $X\sim P$ and $Y\sim Q$ with $P = (1- w)P_0 + wQ_0$ and $Q = (1 - w)Q_0 + wP_0$, where the weight $w = \frac{1}{e^{\epsilon_0} + 1}$, and the distributions $P_0$ and $Q_0$ are defined as $(A+1, C - A)\sim P_0$, and $(A, C - A +1)\sim Q_0$ with $A \sim \mathrm{Binom}(C, 1/2)$ and $C \sim \mathrm{Binom} \mathrm(n-1, 2/(e^{\epsilon_0}+1)).$
It is easy to see that $P_0$ is the  mixture of $\{(A_i+1, i-A_i)\}_{i=0}^{n-1}$ with weights $w_i^0 := \mathbb{P}[C=i]$ and $Q_0$  is the mixture of $\{(A_i, i-A_i +1)\}_{i=0}^{n-1}$ with the same weights.
In this context, $\mathrm{Binom}(k,p)$ is a binomial distribution with parameters $k\in\mathbb{N}$ and $p\in[0,1]$ and each $A_i$ is distributed as $\mathrm{Binom}(i,1/2).$
Advancing our analysis, we adopt the joint concavity, as outlined in Lemma \ref{lemma:inequality}, to establish a lossless bound for the trade-off function $T(P_0,Q_0)$.

\begin{proposition}
\label{prop:shuffling}
Let $F_i$ be the distribution function of $\mathrm{Binom}(i,1/2)$ and let $w_i^0 = \mathbb{P}[C = i]$ for $C \sim \mathrm{Binom}(n-1,2/(e^{\epsilon_0}+1) ).$
Then, we have $T(P_0,Q_0)$ is a piecewise linear function with 
    \begin{align*}
        T(P_0,Q_0) (\alpha(t)) = \sum_{i=0}^{n-1} w_i^0 \left\{ 1 -  F_i\left[F_i^{-1}(\alpha_i(t)) + 1\right]\right\},
    \end{align*}
for each knot $\alpha(t)=\sum_{i=0}^{n-1} w_i^0 \alpha_i(t) := \sum_{i=0}^{n-1} w_i^0 F_i\left(i - \frac{i+1}{t+1} \right).$
\end{proposition}

\textbf{Remark.} Proposition \ref{prop:shuffling} holds with equality and the bound for $T(P_0,Q_0)$ is sharp.

Before stating our results for $T(P,Q)$, we define some notations related to $f$-DP.
For a function $g:\mathbb{R}\rightarrow \mathbb{R}$, let $g^*(y) := \max_{x} \{xy - g(x)\}$ be its convex conjugate.
 For a trade-off function $f$ , let $\mathcal{C}(f) = \min\{f, f^{-1}\}^{**}$ be its symmetrization, where $f^{-1}$ is the left inverse function of $f$, i.e., $f^{-1}\circ f (x) = x$.

 \begin{theorem}
\label{thm:shuffling}
The shuffled mechanism $\mathcal{A}_{\mathrm{Shuffle}}\circ \mathcal{A}$ is $\mathcal{C}(f_{\mathrm{Shuffle}})$-DP. Here $f_{\mathrm{Shuffle}} (\alpha(t))$ is a piecewise linear function 
 where each knot $\alpha(t)$ has the form
    \begin{align*}
        \alpha(t) = \sum_{i=0}^{n-1} w_i^0 \alpha_i(t) := \sum_{i=0}^{n-1} w_i^0 F_i\left(i - \frac{i+1}{t+1} \right)\in[0,1], \qquad \hbox{ for all } t\geq 0,
    \end{align*}
with $F_i$ being the distribution function of $\mathrm{Binom}(i,1/2)$ and $w_i^0 = \mathbb{P}[C = i]$ for $C \sim \mathrm{Binom}(n-1,2/(e^{\epsilon_0}+1) )$,     and the value of $f_{\mathrm{Shuffle}}$ at a knot $\alpha(t)$ is 
    \begin{align*}
    f_{\mathrm{Shuffle}}(\alpha(t)) = 2w \cdot \mathrm{Id}(\alpha(t)) + (1- 2w)\cdot\left[ \sum_{i=0}^{n-1} w_i^0 \left\{ 1 -  F_i\left[F_i^{-1}(\alpha_i(t)) + 1\right]\right\}\right],
    \end{align*}
with $w = \frac{1}{1+e^{\epsilon_0}}$ and $\mathrm{Id}(x) = 1 - x$ being the identity trade-off function.
\end{theorem}

\paragraph{Remark.}
The bound in Theorem \ref{thm:shuffling} is near-optimal. In fact, the proof of Theorem \ref{thm:shuffling} is based on a post-processing procedure in \cite{feldman2023stronger}, joint concavity (Proposition \ref{prop:shuffling}), and advanced joint concavity (Proposition \ref{prop:advancedShuffling}). The post-processing procedure is sharp for specific mechanisms, such as the randomized response mechanism, as shown by Theorem 5.2 and Theorem 5.3 in \cite{feldman2023stronger}. Proposition \ref{prop:shuffling}
holds with equality and is optimal.
The advanced joint concavity, which is an $f$-DP analog of the advanced joint convexity in \cite{balle2018privacy}, is optimal for specific distributions.
Compared to existing analysis of shuffled mechanisms (e.g., \cite{feldman2023stronger}), the main advantage of using 
$f$-DP is that we avoid the use of Hoeffding's inequality and the Chernoff bound to bound the distance between $P_0$ and $Q_0$ in Proposition \ref{prop:shuffling}, which is adopted in \cite{feldman2022hiding,feldman2023stronger} and leads to loose bounds, to bound the mixture of binomial distributions. Moreover, Theorem 3.2 in \cite{feldman2023stronger} holds with an assumption  $\epsilon_0\leq \log\left(\frac{n}{8\log (2/\delta)} - 1 \right)$
, which is removed by using 
$f$-DP in our paper.

To convert $f$-DP to $(\epsilon,\delta)$-DP, we use the primal-dual perspective in \cite{dong2019gaussian} and obtain the following Corollary.

\begin{corollary}
\label{coro:shuffling f to epsilon}
Let 
$
  l\left(t\right) :=-\frac{\sum_{i=0}^{n-1} w_i^0 p_{i}\left(\left\lfloor i +1  - \frac{i+1}{t+1}  \right\rfloor \right)}{\sum_{i=0}^{n-1} w_i^0 p_{i}\left(\left\lfloor i  - \frac{i+1}{t+1}\right\rfloor \right)}
$
with  $p_i$ being the probability mass function of $\mathrm{Binom}(i,1/2)$.
Then, we have
$\mathcal{A}_{\mathrm{Shuffle}}\circ\mathcal{A}$ is $(\epsilon,\delta_{f\text{-DP}}(\epsilon))$-DP for any $\epsilon>0$ with
\begin{align*}
\delta_{f\text{-DP}}(\epsilon) = (- e^{\epsilon}+2w)\left[\sum_{i=0}^{n-1}w_i^0 F_i\left(i - \frac{i+1}{t_{\epsilon}+1}\right) \right] + (1 - 2w)\left[\sum_{i=0}^{n-1}w_i^0 F_i\left(i +1 - \frac{i+1}{t_{\epsilon}+1}\right)\right],
\end{align*}
where $t_{\epsilon} = \inf\{t: -2w + (1 - 2w) l(t) \geq -e^{\epsilon}\}$ and $w = \frac{1}{e^{\epsilon_0} + 1}.$
\end{corollary}

\subsection{Numerical results and comparisons}
\begin{figure}
     \centering
     \begin{subfigure}[b]{0.28\textwidth}
         \centering    \includegraphics[width=\textwidth]{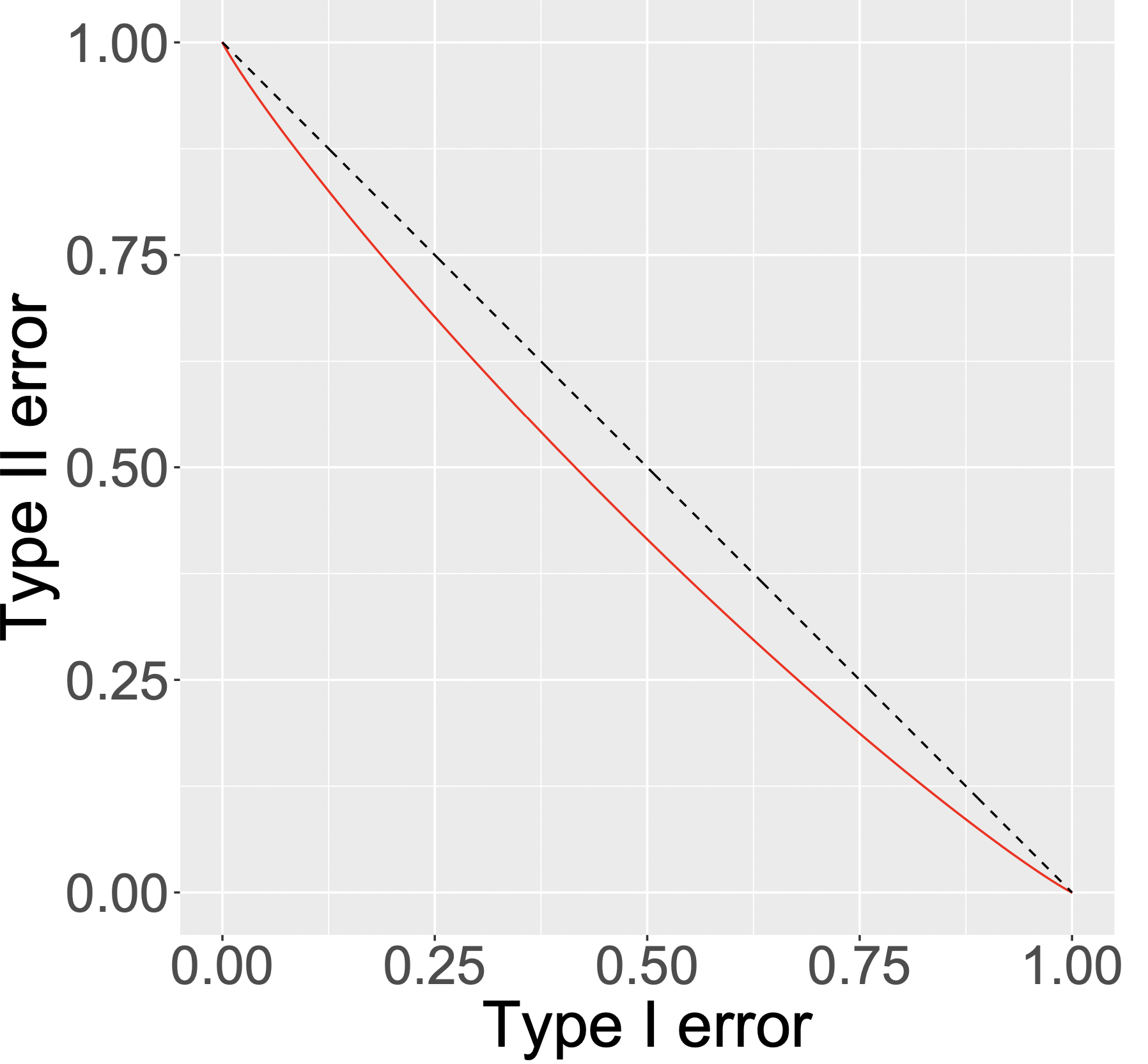}
         \caption{$\epsilon_0=5.444>4.444$}
         \label{fig:epsilon-1}
     \end{subfigure}
     \hfill
     \begin{subfigure}[b]{0.3\textwidth}
         \centering         \includegraphics[width=\textwidth]{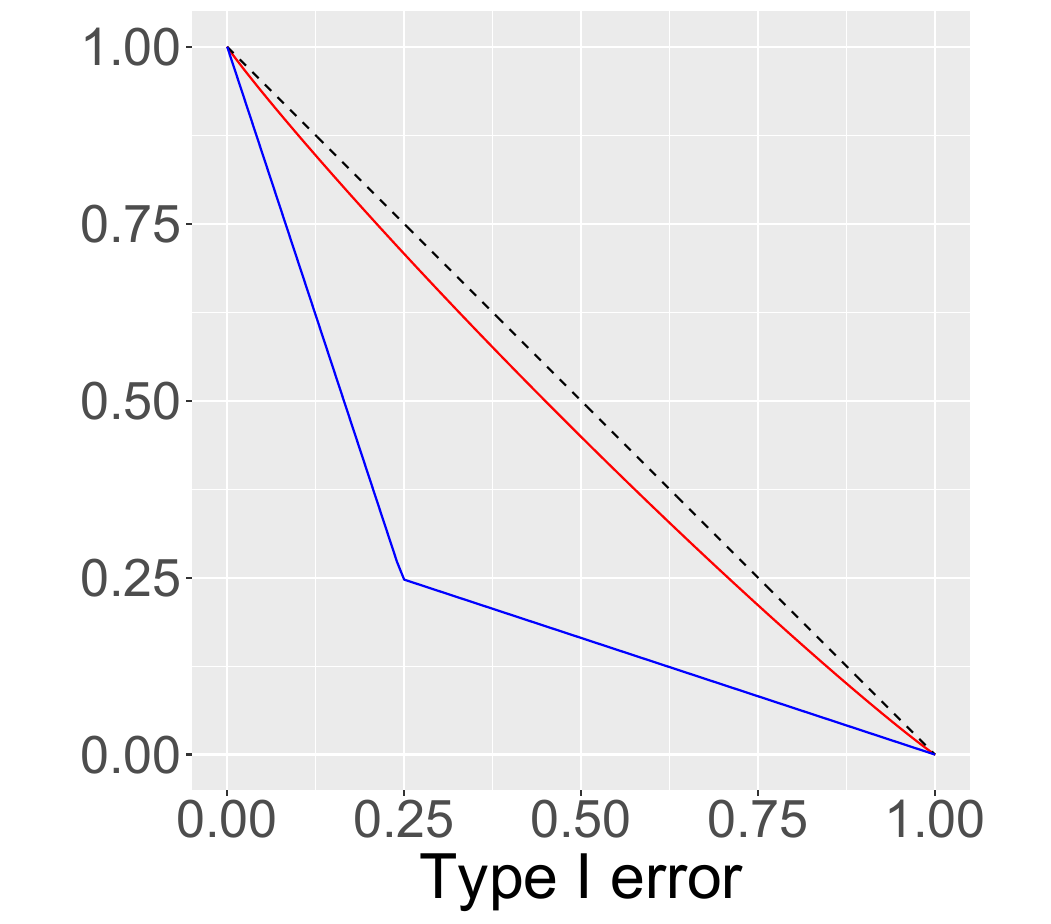}
         \caption{$\epsilon_0=4.444$}
         \label{fig:epsilon0}
     \end{subfigure}
     \hfill
     \begin{subfigure}[b]{0.33\textwidth}
         \centering         \includegraphics[width=\textwidth]{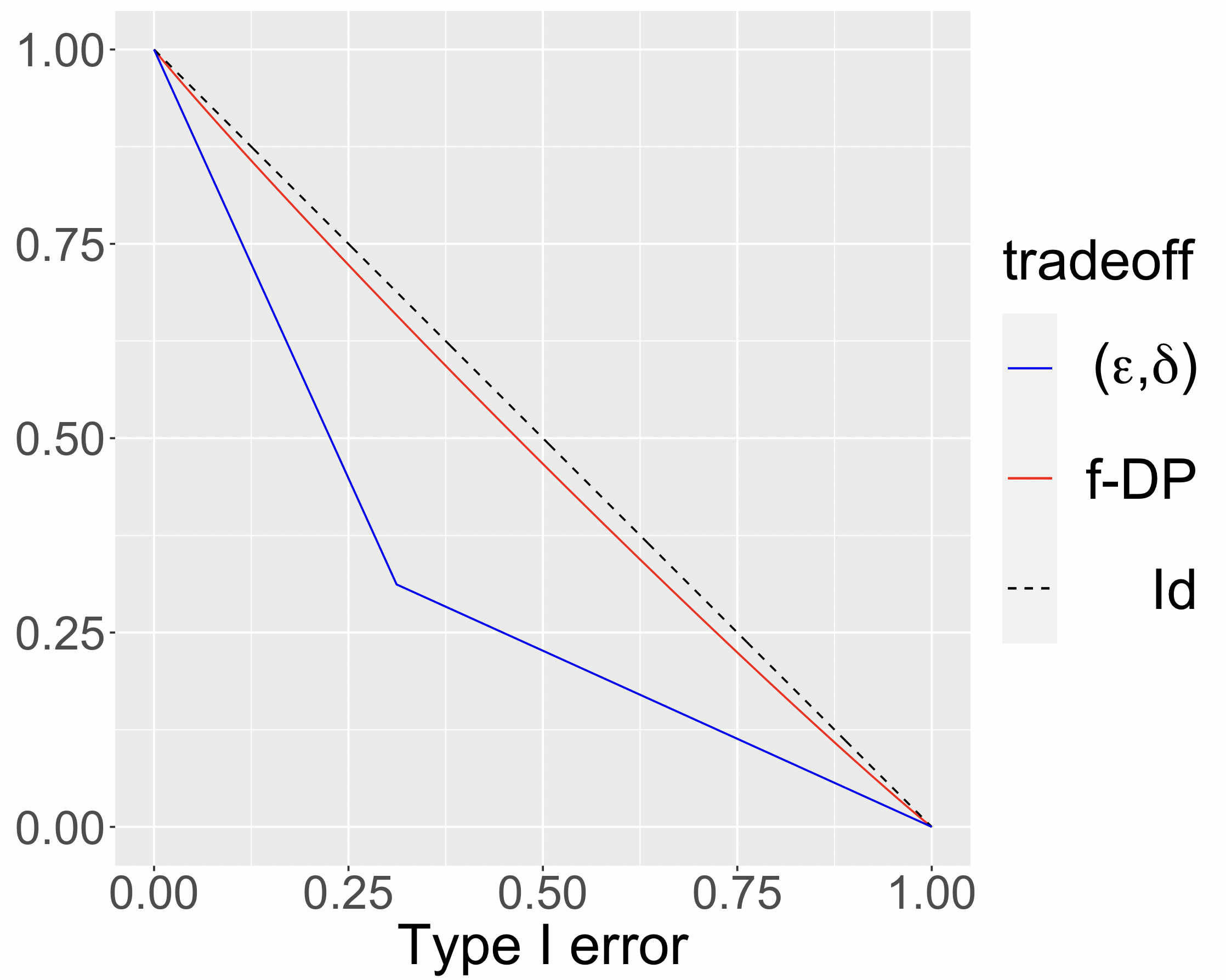}
         \caption{$\epsilon_0 = 3.444$}
         \label{fig:epsilon_0}
     \end{subfigure}
        \caption{A comparison between the trade-off function given by Theorem \ref{thm:shuffling} and $(\epsilon,\delta)$-DP with $n=10000$ and $\delta=n^{-1.5}$ given by \cite{feldman2023stronger}. \cite{feldman2023stronger} requires that $\epsilon_0 \leq \log\left(\frac{n}{8\log (2/\delta)} - 1 \right) \approx 4.444. $ Thus, there is no result for $(\epsilon,\delta)$-DP when $\epsilon_0 >4.444.$ }
    \label{fig:shafflings_small}
\end{figure}
To the best of our understanding, the leading privacy analysis for shuffled mechanisms is given in \cite{feldman2023stronger}.
In this section, we compare the privacy bounds from our Theorem \ref{thm:shuffling} and Corollary \ref{coro:shuffling f to epsilon} with those found in Theorem 3.2 of \cite{feldman2023stronger}. Additionally, we assess the tightness of our bound against the empirical lower bounds obtained through binary search.

Specifically, Figure \ref{fig:shafflings_small} presents a comparison of the trade-off function derived from our Theorem \ref{thm:shuffling} to that of \cite{feldman2023stronger}. This comparison clearly illustrates that $f$-DP offers tighter privacy bounds, given that its trade-off function aligns closer to the identity trade-off function.

In our Table \ref{tab:compareFeldman_10k}, we compare the values of  $\delta_{f\text{-DP}}(\epsilon)$, as derived from Corollary \ref{coro:shuffling f to epsilon} with $\delta(\epsilon)$ in \cite{feldman2023stronger}.
The results indicate that $\delta_{f\text{-DP}}(\epsilon)$ is significantly smaller than $\delta(\epsilon)$.

In Table \ref{tab:compareFeldman_numerical}, we present $\epsilon_{f\text{-DP}}$ alongside the numerical upper bound of $\epsilon$ from \cite{feldman2023stronger} and the numerical lower bound determined by binary search. Given its closeness to the lower bound, our Theorem \ref{thm:shuffling} can be considered near-optimal.

\begin{table}[H]
  \caption{Comparisons with \cite{feldman2023stronger}}
  \label{tab:compareFeldman_10k}
  \medskip
  \centering
  \begin{tabular}{ lllllll  }
\toprule 
 $\epsilon$ & 0.5 & 0.6 & 0.7 &0.8 & 0.9 & 1.0 \\
 \hline
 $\delta$ in \cite{feldman2023stronger} & $0.9494$ & $0.3764$ & $0.1038$  & $0.0181$ & $0.0018$  & $8 \times 10^{-5}$ \\
 \hline
 $\delta_{f\text{-DP}}$ (ours) & $3 \times 10^{-6}$ & $10^{-7}$ & $4 \times 10^{-9}$ & $9 \times 10^{-11}$ & $2 \times 10^{-12}$ & $2 \times 10^{-14}$  \\
  \bottomrule
\end{tabular}
\caption*{
We compare $\delta_{f\text{-DP}}$ obtained in Corollary \ref{coro:shuffling f to epsilon} with the corresponding $\delta$ derived from \cite{feldman2023stronger} using a fixed value of $\epsilon_0 = 4.444$ and $n=10000$. Notably, $\delta_{f\text{-DP}}$ is significantly smaller than $\delta$.
}
\end{table}

\begin{table}[H]
  \caption{Comparisons with numerical results in \cite{feldman2023stronger}} 
  \medskip
  \centering
  \begin{tabular}{ llllll  }
\toprule 
 $\delta$ & $5 \times 10^{-5}$& $3 \times 10^{-6}$ & $10^{-7}$ & $4 \times 10^{-9}$ & $9 \times 10^{-11}$   \\
 \hline
  $\epsilon_{f\text{-DP}}$ (ours) & $0.4$ & 0.5 & 0.6 & 0.7 &0.8 \\
 \hline
 Numerical $\epsilon$ upper bound in \cite{feldman2023stronger} & $1.014$& $1.085$ & $\epsilon_0$ & $\epsilon_0$  & $\epsilon_0$ \\
  \hline
 Numerical $\epsilon$ lower bound & $0.369$& $0.470$ & $0.575$ & $0.664$  & $0.758$\\
  \bottomrule
\end{tabular}
\caption*{
We compare $\epsilon_{f\text{-DP}}$ obtained from Corollary \ref{coro:shuffling f to epsilon} with the corresponding numerical  upper bound $\epsilon$ derived from \cite{feldman2023stronger} using a fixed value of $\epsilon_0 = 4.444$ and $n=10000$.
For $\delta < 10^{-7}$, the bound in \cite{feldman2023stronger} fails as the assumption $\epsilon_0 \leq \log\left(\frac{n}{8\log (2/\delta)} - 1 \right)$ is violated while our theory removes this assumption and holds for all $\epsilon_0$.
Moreover, we compare our theoretical upper bound with the empirical lower bound obtained by binary search in \cite{feldman2023stronger} which shows that our bound is near-optimal.}
\label{tab:compareFeldman_numerical}
\end{table}

In summary, our non-asymptotic privacy bound for shuffled mechanisms outperforms Theorem 3.2 in \cite{feldman2023stronger}. This improvement is a result of our Proposition \ref{prop:shuffling}, which optimally refines Lemma A.4 in \cite{feldman2023stronger}. Besides Proposition \ref{prop:shuffling}, the remainder of our proof of Theorem \ref{thm:shuffling} closely adheres to the methodology presented in \cite{feldman2023stronger}. Our near-optimal result is complicated due to its tightness. Thus, it is difficult to compare our result with the asymptotic bound in \cite{feldman2023stronger} analytically.

\section{Privacy analysis of one-iteration DP-GD with random initialization}
\label{sec:DPGD}
A significant challenge in the privacy analysis of the last-iteration model of DP-SGD lies in accounting for multiple randomization techniques used during iterations. 
This includes aspects like initialization, iterative steps, and sub-sampling. Since these techniques incorporate a mixture of random noise, the joint convexity of $F$-divergence becomes crucial in the privacy analysis of DP-SGD \cite{ye2022differentially,DBLP:conf/nips/AltschulerT22}.
Our Lemma \ref{lemma:inequality}, which provides a unified perspective on these convexity notations, has driven us to include it in the privacy analysis of DP-GD. 
Nevertheless, analyzing the trajectories from multi-step iterations remains complex. 
Therefore, our initial exploration is to investigate the effects of random initialization on a one-step iterate.
It's noteworthy that in machine learning, training a deep neural network using (stochastic) gradient descent combined with random initialization is widely adopted \cite{DBLP:conf/icml/SutskeverMDH13}. The significance of random initialization in noisy gradient descent is also emphasized by \cite{ye2023initialization} within the framework of Kullback-Leibler privacy.

Consider a dataset $\mathcal{D}=\{(x_i,y_i)\}_{i=1}^n$ with $x_i\in\mathbb{R}$ being the features and $y_i\in\mathbb{R}$ being the labels.
Let $\ell(\theta,\mathcal{D})$ be a loss function and let $g(\theta,\mathcal{D})$ be the gradient of $\ell$ with respect to $\theta$.
The output of one-step iteration of
DP-GD initialized at $\theta_0$ with step-size 1 is given by
\begin{align}
\label{eq:1-stepDPGD}
    \theta (\mathcal{D}) = \theta_0 - \left(g(\theta_0,\mathcal{D}) + \mathcal{N}(0,\sigma^2)\right).
\end{align}
In the setting of random initialization, $\theta_0$ is chosen as a Gaussian random variable. Without loss of generality, we consider $\theta_0= I \sim \mathcal{N}(0, 1)$ and rewrite $
    \theta(\mathcal{D}) = s_{I}(\mathcal{D}) + \mathcal{N}(0,\sigma^2)$ with $s_{I}(\mathcal{D}) = I -  g(I,\mathcal{D}).$
$\theta(\mathcal{D})$ is a Gaussian random variable when the initialization $I$ is given, that is, $\theta(\mathcal{D})|I=i\sim \mathcal{N}(s_i(\mathcal{D}),\sigma^2).$
Thus, we can regard $\theta(\mathcal{D})$ as an infinite mixture of Gaussian distributions with continuous Gaussian weights $\{\varphi(i)\}_{i\in\mathbb{R}}$, where $\varphi$ is the pdf of $I$
and the corresponding trade-off function $T(\theta(\mathcal{D}_0),\theta(\mathcal{D}_1)$ can be bounded using the joint concavity.

For simplicity, we define $\theta(\mathcal{D})|I$ as a random variable with a given initialization $I$.
For two neighboring datasets $\mathcal{D}_0$ and $\mathcal{D}_1$, it holds  
\begin{align*}  T(\left(\theta(\mathcal{D}_0)|I, I\right), \left(\theta(\mathcal{D}_1)|I,I\right))=T((X|I,I),(Y|I,I))
\end{align*}
 with $X|I \sim \mathcal{N}(0,1)$ and $Y|I \sim\mathcal{N}(\mu_I,1)$ for $I\sim\mathcal{N}(0,1)$,
where $\mu_I = (g(I,\mathcal{D}_1) - g(I,\mathcal{D}_0))/\sigma.$

\begin{theorem}
\label{thm:DPGD}
Let $\theta(\mathcal{D}_0)$ and $\theta(\mathcal{D}_1)$ be defined in \eqref{eq:1-stepDPGD} for neighboring datasets $\mathcal{D}_0 $ and $\mathcal{D}_1 $. Then, we have
\begin{align*}
T(\theta(\mathcal{D}_0),\theta(\mathcal{D}_1)) (\alpha(t))\geq\mathbb{E}_I\left[ \Phi(-t_I + \mu_I)\cdot\mathds{1}_{[\mu_I\leq 0]} +  \Phi(t_I - \mu_I)\cdot\mathds{1}_{[\mu_I> 0]}\right] 
\end{align*}
    with $t_I = -\frac{t}{\mu_I} + \frac{\mu_I}{2}$ and 
$
      \alpha(t) =\mathbb{E}_I\left[ \Phi(t_I)\cdot\mathds{1}_{[\mu_I\leq 0]} +  \Phi(-t_I)\cdot\mathds{1}_{[\mu_I> 0]}\right].
$
Here $\Phi$ is the cumulative distribution function of $\mathcal{N}(0,1)$ and the expectation is taken with respect to $I$.
\end{theorem}

\textbf{Remark.} Note that Theorem \ref{thm:DPGD} is instance-based privacy guarantee as it relies on the datasets.
To extend it to the worst case, we let $\mu_I^{\mathrm{max}} = \max_{\mathcal{D}_0,\mathcal{D}_1}\left\{\left|g(I,\mathcal{D}_1^{\max}) - g(I,\mathcal{D}_0^{\max})\right|/\sigma\right\}$ be the sensitivity of the gradient with a given initialization $I$.
As a result, $\theta(\mathcal{D})$ output by one-step  DP-GD is $f$-DP with $f(\alpha(t))=\mathbb{E}_I\left[  \Phi(t_I^{\mathrm{max}} - \mu_I^{\mathrm{max}})\right],
$
where $t_I^{\mathrm{max}} = -\frac{t}{\mu_I^{\mathrm{max}}} + \frac{\mu_I^{\mathrm{max}}}{2}$ and 
$
      \alpha(t) =\mathbb{E}_I\left[ \Phi(-t_I^{\mathrm{max}})\right].
$
The worst case trade-off function is bounded for strongly convex loss functions with a bounded data domain.

To numerically evaluate the trade-off function in Theorem \ref{thm:DPGD}, we consider an example $\mathcal{D}_0=\{(x_i,y_i)\}_{i=1}^n$ 
with $y_i=a x_i$ and $x_i^2 = 1$ for some constant $a$ and we defined $\mathcal{D}_1$ by removing an arbitrary element in $\mathcal{D}_0$.
Moreover, we assume that $\sigma=1.$
Note that for this example without gradient clipping, the gradient is linear in $I$ and $\theta(\mathcal{D}_0)$ is the sum of two Gaussian random variables which is Gaussian.
Thus, the trade-off function has a closed-form representation.
In general, the output is non-Gaussian and we should adopt Theorem \ref{thm:DPGD}. 
For example, if we consider gradient clipping \cite{abadi2016DPSGD, Bu2020Deep} and replace $g(\theta,\mathcal{D})$ by the clipped gradient
\begin{align*}
    g_c(\theta,\mathcal{D}) = \sum_{i=1}^n \frac{g^{(i)}(\theta)}{\max\{1, \|g^{(i)}(\theta)\|_2/c\}}, \qquad \hbox{with } g^{(i)}(\theta) = (y_i - \theta x_i) (-x_i),
\end{align*}
 where the gradient of each data point $g^{(i)}$ is cut off by some constant $c > 0$,
 then $\mu_I^{\max}$ is given by
\begin{align*}
    \mu_{I}^{\max} = \left\{
     \begin{array}{cc}
        a - I,  & |a - I| \leq c ,  \\
       c,   & a - I \geq c, \\
       -c, & a - I\leq -c,
     \end{array}
    \right.
\end{align*}
which is not Gaussian.
In this example $g_c(\theta,\mathcal{D}) + \mathcal{N}(0,1)$ is considered as $c$-GDP  if we disregard the effects of random initialization since the sensitivity of $g_c$ is $c$.

We illustrate the trade-off function of Theorem \ref{thm:DPGD} computed numerically in Figure \ref{fig:DPGD_clip}, where we also compare it with $c$-GDP for $a=1$ and varying values of $c$. 
Overall, the figure suggests that random initialization can amplify the privacy of DP-GD, as our bounds outperform those of $c$-GDP, which does not take into account the randomness of initialization.
Furthermore, we observe that as $c$ increases, the amplification effect caused by random initialization becomes more significant, since the difference between $T((X|I,I), (Y|I,I))$ and $c$-GDP also increases. This is reasonable, since the randomness resulting from initialization comes from $I$ such that $|a - I| \leq c$, whereas for $|a - I| > c$, $\mu_I$ remains constant and no randomness is introduced.
Thus, the random initialization introduces greater levels of randomness as $c$ increases.

It is worth noting that in this example, without gradient clipping, we have $\mu_I^{\max} = a - I$ and the dominate pair are two Gaussian distributions $\mathcal{N}(0,1)$ and $\mathcal{N}(0,2)$.
The R\'enyi DP fails to measure the privacy of initialization.
In fact, it holds $R_{\widetilde{\alpha}}(\mathcal{N}(0,1) \| \mathcal{N}(0,2) ) = \infty$ for $\widetilde{\alpha}$ large enough.

\begin{figure}[ht]
     \centering
          \begin{subfigure}[b]{0.3\textwidth}
         \centering         \includegraphics[width=\textwidth]{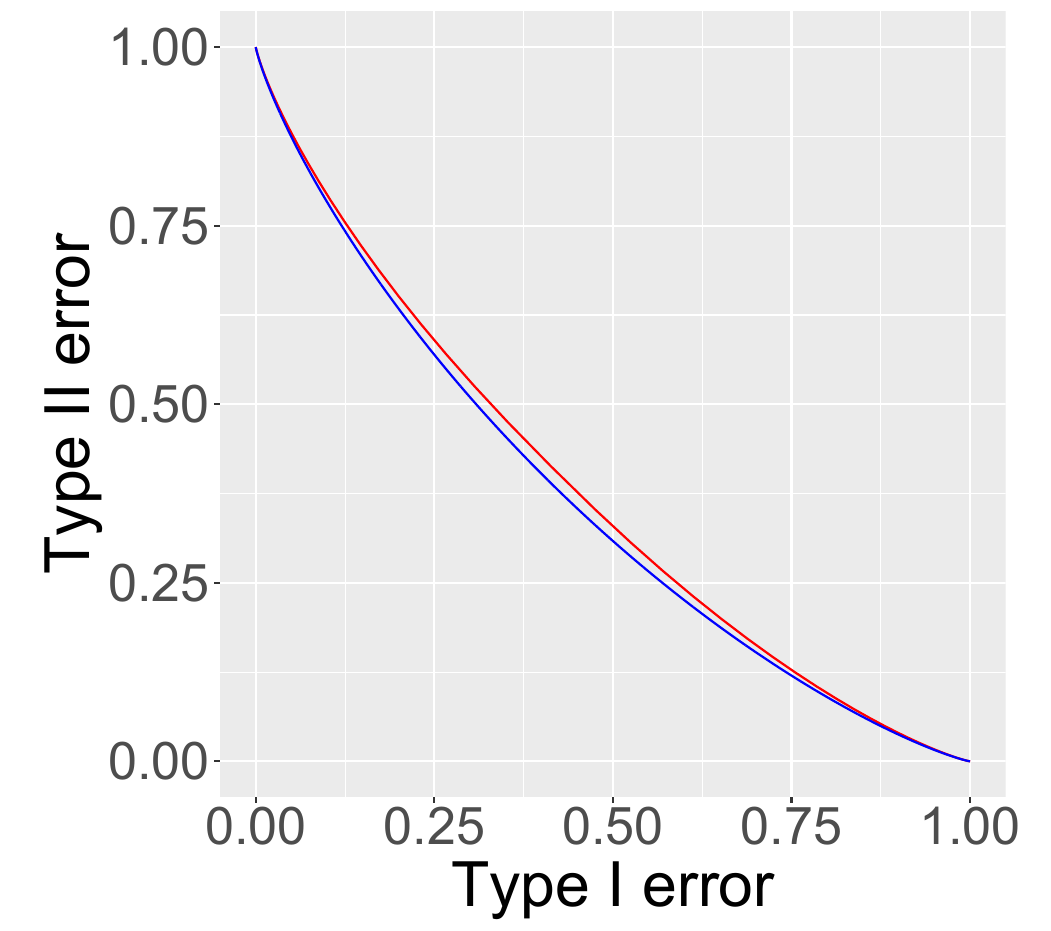}
         \caption{$c=0.5$}       \label{fig:DPGD_clip_0.5}
     \end{subfigure}
     \hfill
     \begin{subfigure}[b]{0.3\textwidth}
    \centering           \includegraphics[width=\textwidth]{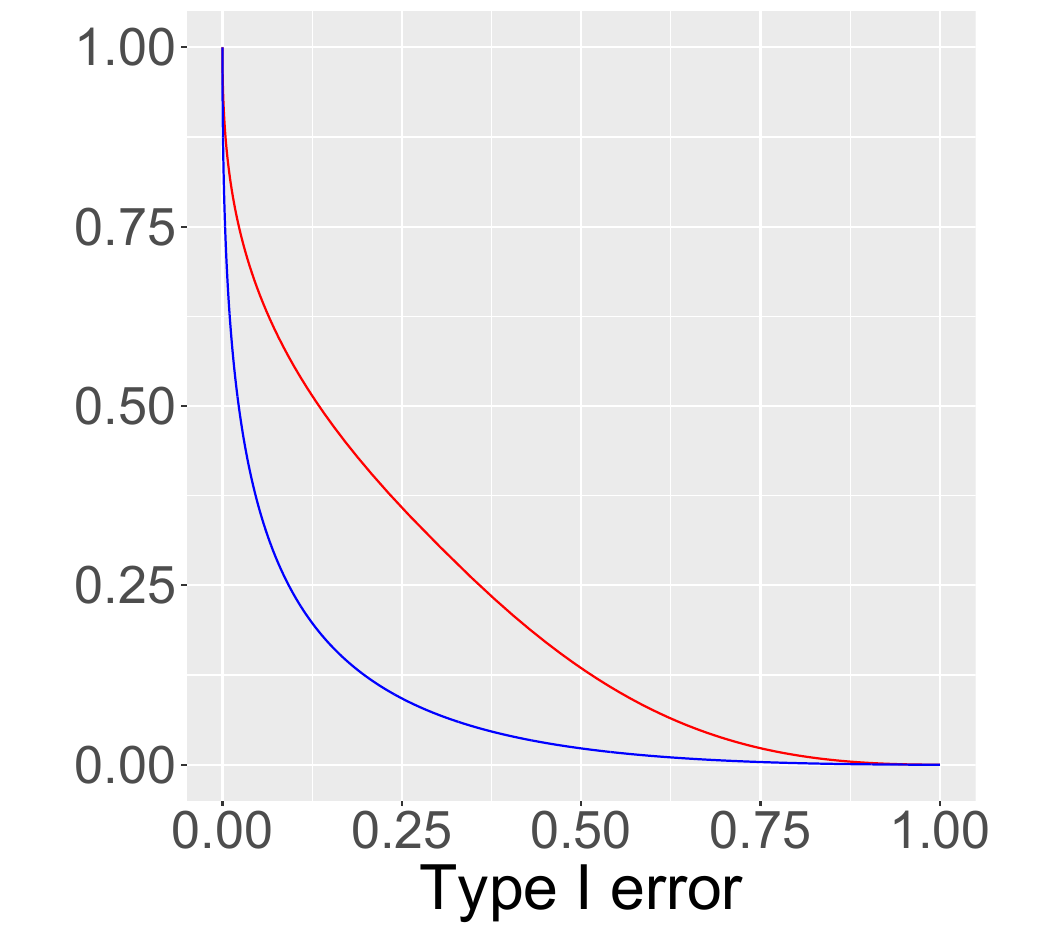}
    \caption{$c=2$}
    \label{fig:DPGD_clip_gdp}
     \end{subfigure}
     \hfill
     \begin{subfigure}[b]{0.36\textwidth}
         \centering         \includegraphics[width=\textwidth]{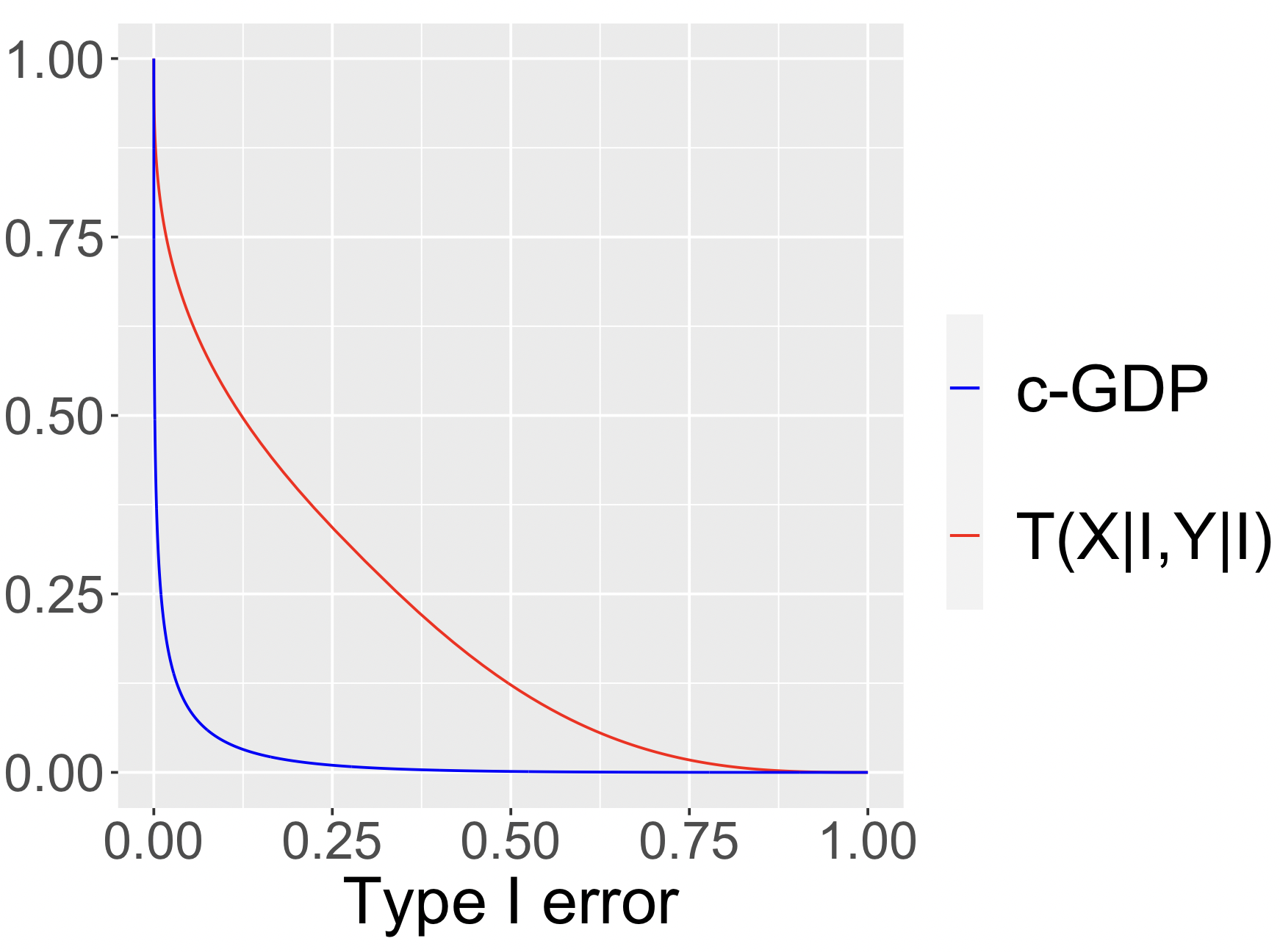}
         \caption{$c=3$}       \label{fig:c-gdp}
     \end{subfigure}
    \caption{Trade-off functions for linear models with $a=1$. }
    \label{fig:DPGD_clip}     
\end{figure}

\section{Optimality of joint concavity and advanced joint concavity}
\label{sec:theory}
In this section, we first explore the sufficient and necessary conditions under which Lemma \ref{lemma:inequality} holds with equality. While Lemma \ref{lemma:inequality} is generally not sharp, we introduce an $f$-DP analog of the advanced joint convexity of the hockey-stick divergence from \cite{balle2018privacy}, yielding tighter bounds in certain applications.

\begin{figure}
     \centering
          \begin{subfigure}[b]{0.45\textwidth}
         \centering       \includegraphics[width=0.7\textwidth]{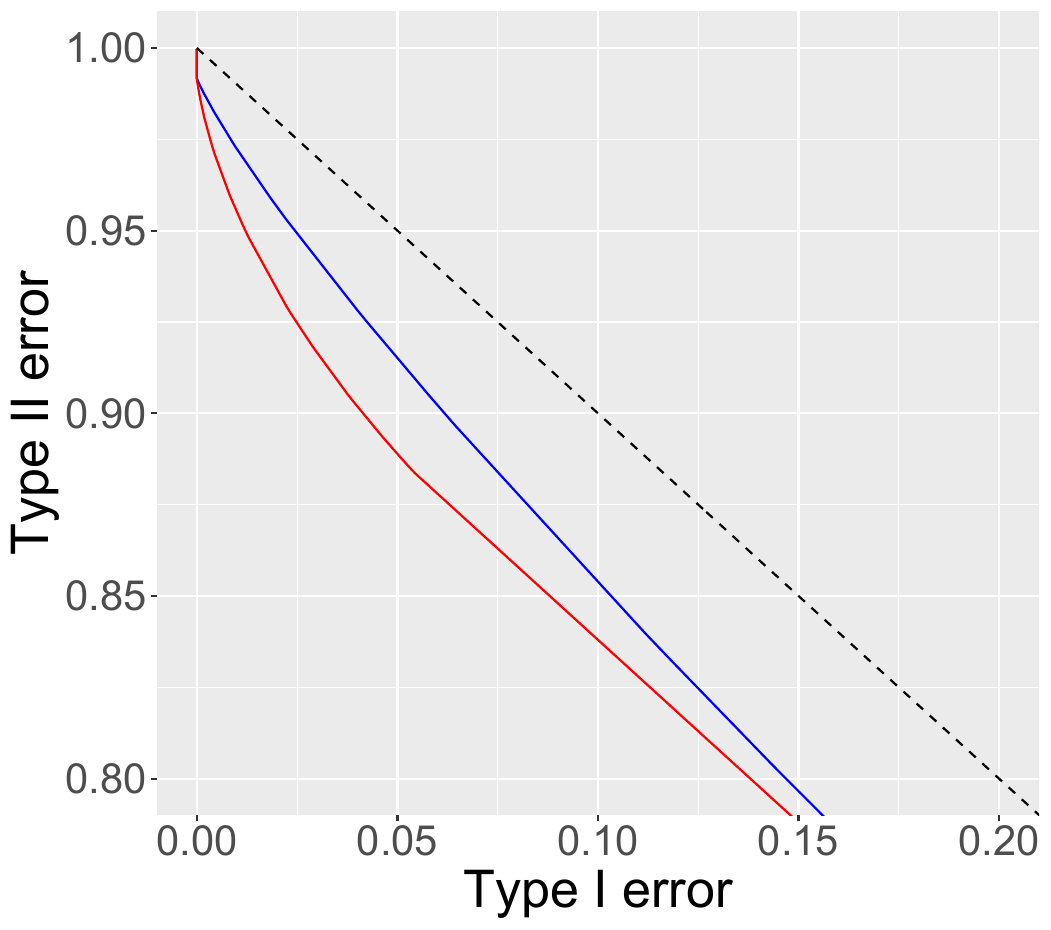}
         \caption{Top-left corner of Figure \ref{fig:compare_shuffling} for small type I error $\alpha$. As the parameters $(\epsilon,\delta)$ computed by the trade-off function only depend on small $\alpha$, advanced joint concavity leads to a tighter bound.}
     \end{subfigure}
     \hfill
     \begin{subfigure}[b]{0.45\textwidth}
         \centering    \includegraphics[width=\textwidth]{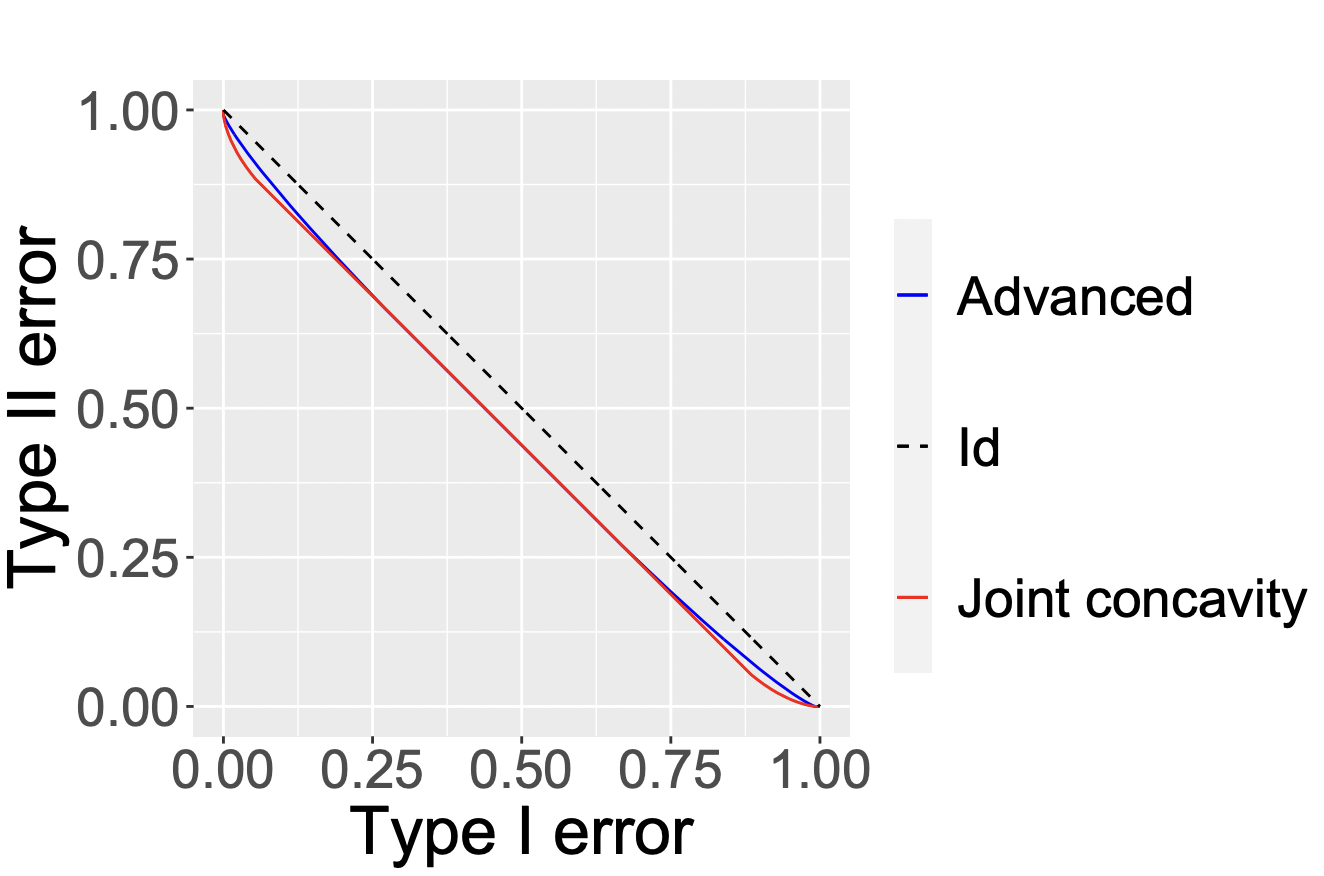} \caption{Trade-off functions obtained by (advanced) joint concavity. For example, for $w=1/3$ and $\epsilon=0.5$, $\delta$  derived from Lemma \ref{lemma: f advanced joint convexity} is $1.5 \times 10^{-6}$ while that from Lemma \ref{lemma:inequality} is $0.0020.$
         }        \label{fig:compare_shuffling}
     \end{subfigure}
     \hfill
    \caption{Comparison between joint concavity (Lemma \ref{lemma:inequality}) and advanced joint concavity (Lemma \ref{lemma: f advanced joint convexity}).}
        \label{fig:Tightness} 
\end{figure}

Recall the distributions $P = (1 - w) P_0 + w Q_0$ and $Q = (1 - w) Q_0 + w P_0$ that appear in the shuffled mechanisms.
Bounding the trade-off function $T(P,Q)$ directly using the joint concavity leads to a loose bound (cf., Figure \ref{fig:compare_shuffling}).
For the scenarios where Lemma \ref{lemma:inequality} is not tight, we introduce the $f$-DP analog of the advanced joint convexity of $(\epsilon,\delta)$-DP \cite{balle2018privacy} that may lead to tighter bounds and we term it the "advanced joint concavity of trade-off functions".

The following proposition presents a necessary and sufficient condition for Lemma \ref{lemma:inequality} to hold with equality.

\begin{proposition}
   \label{prop:iff-equality}
    For $m=2$,
    Lemma \ref{lemma:inequality} holds with equality if and only if  $\frac{w_1 p_1 + w_2 p_2}{w_1 q_1 + w_2 q_2} (X) \overset{\mathbb{P}}{=} w_1 \frac{p_1}{q_1}(X) + w_2 \frac{p_2}{q_2}(X)
$
 with $X\sim P_{\bfw}$, where for $p_i(X)/q_i(X) = 0/0$ and $p_j(X)/q_j(X) \neq 0/0$ with $i\neq j$, we set $p_i(X)/q_i(X) = p_j(X)/q_j(X)$.
\end{proposition}
It is not difficult to see that $P_0$ and $Q_0$ in shuffling models satisfy this necessary and sufficient condition when $n=2$.

As we discussed, Lemma \ref{lemma:inequality} may not be sharp in general.
The following lemma is about the advanced joint convexity of the hockey-stick divergence, which is a slight generalization of Theorem 2 in \cite{balle2018privacy}.

\begin{lemma}
\label{lemma: HS advanced joint convexity}
For any non-negative $\epsilon', \epsilon_0,\epsilon_1, \gamma,$ and $ \eta$ satisfying
$
    \exp(\epsilon') = (1 - w) \exp(\epsilon_0) + w \exp(\epsilon_{1})
$
and 
$
    \exp(\epsilon_0) (1 - w) \gamma + \exp(\epsilon_1) w \eta = \exp(\epsilon') w,
$
we have
\begin{align}\label{eq:HS-advanced}
\begin{split}
    &H_{e^{\epsilon'}} \left((1 - w) P_{1} + w P_{2} \| (1-w) Q_{1} + w Q_{2} \right)\\
   & \leq\  (1 - w) H_{e^{\epsilon_0}} \left(P_1 \| (1-\gamma) Q_{1} + \gamma Q_{2} \right) + w H_{e^{\epsilon_1}} \left(P_{2} \| (1-\eta) Q_{1} + \eta Q_{2} \right).
\end{split}
\end{align}
\end{lemma}
Lemma \ref{lemma: HS advanced joint convexity} is reduced to the advanced joint convexity of the hockey-stick divergence in \cite{balle2018privacy}
when $P_1 = Q_1,$ by minimizing the right-hand-side of \eqref{eq:HS-advanced} with respect to $\gamma, \eta, \epsilon_0,$ and $\epsilon_1$.

Recall the convex conjugate $g^*$ of a function $g$ defined by $g^*(y) = \sup_x \{xy - g(x)\}$ and $\mathcal{C}(f) = \min\{f,f^{-1}\}^{**}$ which is the symmetrization of $f$.
We have the following advanced joint concavity of trade-off functions.

\begin{lemma}[Advanced joint concavity] 
\label{lemma: f advanced joint convexity}
Suppose that $T(P_i,Q_i)$ is symmetric for each $i$. Then, for $0\leq w\leq 1$, we have
    \begin{align*}
        T((1 - w) &P_{1} + w P_{2}, (1-w) Q_{1} + w Q_{2}) \\
        &\geq 
        \mathcal{C} \Bigg(\bigg((1 - w)(1 - \gamma) F_{1,1}^* + w(1 - \eta) F_{2,1}^* + (1 - w) \gamma F_{1,2}^{*} + w \eta F_{2,2}^{*}  \bigg)^{*} \Bigg)
    \end{align*}
    for arbitrary $0 \leq \gamma < w < \eta \leq 1$, where $F_{i,j}(x)$ is given by 
    $ F_{1, i}(x) := f_{1,i} \left(\frac{x (1 - w) (\eta - \gamma)}{(\eta - w)} \right),$ and 
    $F_{2, i}(x) := f_{2,i} \left(\frac{ x w (\eta - \gamma)}{(w - \gamma)}\right),$ 
    and the trade-off functions are defined as$
        f_{i,j} = T(P_i, Q_j)$  for $1 \leq i,j \leq 2$.
Moreover, for $\gamma = \eta = w$, it holds
    \begin{align*}
        T((1 - w) P_{1} + w P_{2}, &(1-w) Q_{1} + w Q_{2}) 
        \\
        &\geq \mathcal{C}\left((1 - w) T \left(P_1 , (1 - w) Q_1 + w Q_2 \right) + w T \left(P_2 ,(1 - w) Q_1 + w Q_2 \right)\right).
    \end{align*}
\end{lemma}

Determining the trade-off functions using advanced joint concavity can be challenging in many practical situations.
In fact, to apply the advanced joint concavity, one need to specify the choice of $\gamma,\eta$ by maximizing the right-hand-side of Lemma \ref{lemma: f advanced joint convexity}.
Therefore, in real-world applications, we often rely on both joint concavity and advanced joint concavity.

For $P=(1- w) P_0 + wQ_0$ and $Q=(1- w) Q_0 + wP_0$ in shuffling models, we have the following bound derived from Lemma \ref{lemma: f advanced joint convexity}.

\begin{proposition}
\label{prop:advancedShuffling}
For $P=(1- w) P_0 + wQ_0$ and $Q=(1- w) Q_0 + wP_0$ with some weight $0\leq w \leq 1/2$, we have 
$
T(P,Q) \geq \mathcal{C} \left( 2w \mathrm{Id} + (1 - 2w) T(P_0,Q_0) \right).
$
\end{proposition}
The equality in Proposition \ref{prop:advancedShuffling} does not hold exactly. However, this lower bound is  almost the tightest closed-form expression. One may refer to Section \ref{sec:proof_prop_advanced_shuffling} in the appendix for the proof details.

\section{Discussion}

This paper provides refined privacy bounds for mixture mechanisms, including shuffling models and DP-GD with random initialization. For shuffling models, we present a bound that is tighter than existing results based on $(\epsilon,\delta)$-DP. In the study of DP-GD, we demonstrate how random initialization can amplify privacy concerns. These bounds are derived using a unified $f$-DP approach based on the joint concavity and advanced joint concavity of trade-off functions. We also investigate the sharpness and other properties of these concavity notions.

In our future work, we plan to extend our analysis from one-step DP-GD to multi-step DP-SGD. For DP-SGD with multiple iterations, it is crucial to consider subsampling and privacy amplification by iteration in the privacy accountant, in addition to the randomness introduced by shuffling and random initialization. While there is an $f$-DP bound for subsampling provided in an independent work \cite{wang2022differentially}, as far as we know, there is limited research on $f$-DP results regarding privacy amplification by iteration.

Beyond DP-SGD, we intend to extend our theory to the privacy analysis of other key applications that involves various randomization techniques. These include the shuffled Gaussian mechanism for federated learning, as discussed in \cite{DBLP:conf/aistats/GirgisDDKS21}, and the composition of mixture mechanisms. For extending our theory to federated learning, we might adopt the $f$-DP framework outlined in \cite{federatedfDP}. Addressing the composition of mixture mechanisms demands examination of the tensor product between trade-off functions. This is a complex task, even when dealing with the simplest mixture mechanisms like sub-sampling, as highlighted in \cite{DBLP:conf/aistats/0005DW22}.

\section*{Acknowledgments}

Weijie J.~Su was supported in part by a Meta Research Award and NSF through CCF1934876.

Reza Shokri was supported by a Google PDPO Faculty Research Award, Intel within the www.private-ai.org center, a Meta Faculty Research Award, the NUS Early Career Research Award (NUS ECRA award number NUS ECRA FY19 P16), and the National Research Foundation, Singapore under its Strategic Capability Research Centres Funding Initiative.

{
\small
\bibliographystyle{plain} 
\bibliography{Ref.bib}
}


\newpage
\appendix

\section{Essential foundations of differential privacy}
The Neyman-Pearson lemma (cf., \cite{MR2135927}) is crucial in the proof as it establishes the likelihood ratio test as the most powerful test.

\begin{lemma}[Neyman-Pearson]
\label{lemma:NP}
    Let $P$ and $Q$ be probability distributions on $\Omega$ with densities $p$ and $q$, respectively. For the hypothesis testing problem $H_0 : P$ vs $H_1 : Q$, a test $\phi : \Omega \to [0, 1]$ is the most powerful test at level $\alpha$ if and only if there are two constants $t \in [0, + \infty]$ and $c \in [0, 1]$ such that $\phi$ has the form
    \begin{align*}
        \phi(\omega) = \left\{ 
        \begin{array}{ll}
        1,     &     \text{if}\ p(\omega) < t q(\omega),  \\
        c,     &     \text{if}\ p(\omega) = t q(\omega),  \\
        0,     &     \text{if}\ p(\omega) > t q(\omega).  \\
        \end{array} \right.
    \end{align*}
\end{lemma}
As an application of the Neyman-Pearson lemma, the type I error $\alpha(t)$ has the form 
\begin{align*}
    \alpha(t) = \mathbb{E}_{P}[\phi] = \mathbb{P}_{X\sim P}\left[\frac{p(X)}{q(X)}<t \right] + c\mathbb{P}_{X\sim P}\left[\frac{p(X)}{q(X)}=t \right],
\end{align*}
and the type II error is
\begin{align*}
    \beta(t) = 1 - \mathbb{E}_Q[\phi] = \mathbb{P}_{X\sim Q}\left[\frac{p(X)}{q(X)}>t \right] + (1 - c) \mathbb{P}_{X\sim Q}\left[\frac{p(X)}{q(X)}=t \right].
\end{align*}

One of the most important properties of differential privacy is that DP is immune to data-independent post-processing. Precisely, we introduce the following information processing inequality given by \cite{dong2019gaussian}.
\begin{lemma}[Theorem 2.10 in \cite{dong2019gaussian}]
\label{lemma:processing}
    Let $P$ and $Q$ be two distributions on a probability space $\mathcal{Z}$ and let $\widetilde{P}$ and $\widetilde{Q}$ be two distributions on another probability space $\widetilde{\mathcal{Z}}.$
    The following two statements are equivalent:
    \begin{itemize}
        \item[(a)] $T(P,Q) \leq T(\widetilde{P},\widetilde{Q})$.
        \item[(b)] There exists a post-processing algorithm $\mathrm{Proc}:\mathcal{Z}\rightarrow \widetilde{\mathcal{Z}}$ such that $\mathrm{Proc}(P) = \widetilde{P}$ and $\mathrm{Proc}(Q)=\widetilde{Q}.$
    \end{itemize}
\end{lemma}

The primal-dual perspective, initially introduced by \cite{dong2019gaussian}, will be employed to explore the relationship between $(\epsilon,\delta)$-DP and $f$-DP.
Recall that for a function $g$, its convex conjugate $g^{*}$ is defined by $g^{*}(y) = \sup_x\{xy - f(x)\}.$
\begin{lemma}[Proposition 2.12 in \cite{dong2019gaussian}]
    Let $f$ be a symmetric trade-off function. A mechanism is $f$-DP if and only if  it is $(\epsilon,\delta(\epsilon))$-DP for all $\epsilon>0$ with $\delta(\epsilon) = 1 + f^*(-e^{\epsilon}).$
\end{lemma}
To make use of Lemma \ref{lemma: f advanced joint convexity}, we recall the symmetrization of a trade-off function as defined in Definition F.1 of \cite{dong2019gaussian}. Let $f$ be a trade-off function, the symmetrization of $f$ is given by
\begin{align*}
    \mathrm{Symm}(f) = \left\{\begin{array}{cc}
      \min\{f,f^{-1}\}^{**},   &  \hbox{ if 
 }\bar{x}\leq f(\bar{x}),\\
       \max\{f,f^{-1}\},  & \hbox{ if } \bar{x} > f(\bar{x}),
    \end{array}
    \right.
\end{align*}
with $\bar{x} = \inf\{x\in[0,1]:-1\in\partial f(x)\}.$

According to Section F in \cite{dong2019gaussian}, we have
\begin{align}
\label{eq:double-conjugate}
    \min\{f,f^{-1}\}^{**} = \left\{
    \begin{array}{lll}
      f(x),   & 0\leq x \leq \bar{x},  \\
      \bar{x} - f(\bar{x}) -x,   & \bar{x} \leq x \leq f(\bar{x}),\\
     f^{-1}(x), & f(\bar{x}) \leq x \leq 1.
    \end{array}
    \right.
\end{align}

Another useful tool is the advanced joint convexity first introduced in \cite{balle2018privacy}.

\begin{lemma}[Theorem 2 in \cite{balle2018privacy}]
\label{lemma:advanced_balle}
    Let $P$ and $Q$ be two distributions such that $P = (1 - w) P_0 + w P_1$ and $Q = (1 -w ) P_0 + w Q_1$, for some $0\leq w \leq 1.$ Given $\gamma \geq 1$, let $\gamma' = 1 + w(\gamma - 1)$ and $\eta = \gamma'/\gamma.$ Then, it holds
    \begin{align*}
        H_{\gamma'} (P\| Q) = w H_{\gamma}(P_1\|(1 - \eta)P_0 + \eta Q_1).
    \end{align*}
\end{lemma}

\section{Technical details of Section \ref{sec:joint-convex}}

In this section, we discuss the omitted details of Section \ref{sec:joint-convex}. Prior to delving into the proofs, we provide a reminder of the notations.
Let $\{P_i\}_{i=1}^m$ and $\{Q_i\}_{i=1}^m$ be two sequences of probability distributions. 
For a weight vector $\bfw = (w_1,\cdots w_m)$, let $P_\bfw=\sum_{i=1}^m w_i P_i$ and let $Q_\bfw=\sum_{i=1}^m w_i Q_i$.
Let $I$ be a random variable such that $\mathbb{P}[I=i] = w_i.$

\subsection{Proof of Lemma \ref{lemma:inequality} and discussions}

\begin{proof}[Proof of Lemma \ref{lemma:inequality}]
Consider $X\sim P_{\bfw}$ and $Y\sim Q_{\bfw}$. Here, $(X|I, I)$ denotes the observation of $X|I=i\sim P_i$ along with an index $i$, indicating that $X$ is drawn from the $I$-th distribution $P_I$. Therefore, $(X|I, I) \rightarrow X$ represents a post-processing step where we remove the information about $I$.
Since we solely manipulate the indices, this post-processing is independent of the data, leading to the inequality $T(X,Y) \geq T((X|I, I), (Y|I, I))$.

The next step is to specify $T((X|I, I),(Y|I, I)).$
Let $p_I$ and $q_I$ be the pdfs of $P_I$ and $Q_I$, respectively.
According to Lemma \ref{lemma:NP}, the most powerful test is the likelihood ratio test.
Then, the type I error is
\begin{align*}
    \alpha(t,c ) &= \mathbb{P}_{X\sim P_I,I}\left[\frac{p_I(X)}{q_I(X)} < t \right] + c\mathbb{P}_{X\sim P_I,I}\left[\frac{p_I(X)}{q_I(X)} = t \right]
    \\
    &= \mathbb{E}_{I}\left[\mathbb{P}_{X\sim P_I}\left[\left.\frac{p_I(X)}{q_I(X)} < t \right|I \right] + c\mathbb{P}_{X\sim P_I}\left[\left.\frac{p_I(X)}{q_I(X)} = t \right| I\right] \right]
    \\
    &= \sum_{i=1}^m w_i \alpha_i(t,c),
\end{align*}
where
\begin{align*}
    \alpha_i(t,c) &= \mathbb{P}_{X\sim P_I}\left[\left.\frac{p_I(X)}{q_I(X)} < t \right|I=i \right] + c\mathbb{P}_{X\sim P_I}\left[\left.\frac{p_I(X)}{q_I(X)} = t \right| I=i\right]
    \\
    &= \mathbb{P}_{X\sim P_i}\left[\frac{p_i(X)}{q_i(X)} < t \right] + c\mathbb{P}_{X\sim P_i}\left[\frac{p_i(X)}{q_i(X)} = t \right].
\end{align*}
Similarly, the type II error is 
\begin{align*}
    \beta(t,c) &= \mathbb{E}_{I}\left[\mathbb{P}_{X\sim Q_I}\left[\left.\frac{p_I(X)}{q_I(X)} > t \right|I \right] + (1 - c)\mathbb{P}_{X\sim Q_I}\left[\left.\frac{p_I(X)}{q_I(X)} = t \right| I\right] \right]
    \\
    &=\sum_{i=1}^m w_i \left(\mathbb{P}_{X\sim Q_i} \left[\frac{p_i(X)}{q_i(X)} > t \right] + (1-c)\mathbb{P}_{X\sim Q_i}\left[\frac{p_i(X)}{q_i(X)} = t \right] \right) \\
    &=:\sum_{i=1}^m w_i \beta_i(t,c).
\end{align*}
We complete the proof by noting that $\beta_i = T(P_i,Q_i)(\alpha_i).$
\end{proof}

\textbf{Remark.} Lemma \ref{lemma:inequality} can be extended to continuous weights $I$. In fact, for $I$ being a random variable with pdf $\varphi,$ one still has $T(X,Y) \geq T((X|I, I),(Y|I, I))$ using the same post-processing.
Then, we have
\begin{align}
\label{eq:type_I_cts_inequality}
    \alpha(t,c) = \mathbb{E}_{I}\left[\mathbb{P}_{X\sim P_I}\left[\left.\frac{p_I(X)}{q_I(X)} < t \right|I \right] + c\mathbb{P}_{X\sim P_I}\left[\left.\frac{p_I(X)}{q_I(X)} = t \right| I\right] \right]
\end{align}
and
\begin{align}
\label{eq:type_II_cts_inequality}
 \beta(t,c) &= \mathbb{E}_{I}\left[\mathbb{P}_{X\sim Q_I}\left[\left.\frac{p_I(X)}{q_I(X)} > t \right|I \right] + (1 - c)\mathbb{P}_{X\sim Q_I}\left[\left.\frac{p_I(X)}{q_I(X)} = t \right| I\right] \right],
\end{align}
which are non-elementary integrals.
This continuous analog of Lemma \ref{lemma:inequality} will be used to prove Theorem \ref{thm:DPGD}.

Lemma \ref{lemma:inequality} can be extended to $P_{\bfw}$ and $Q_{\bfw'}$ with different weights $\bfw$ and $\bfw'$.
\begin{proposition} \label{coro:diff-weight}
    Let $\bfw = (w_1, w_2)$ and $\bfw' = (w'_1, w'_2)$. It holds
    \begin{align*}
        T(P_{\bfw}, Q_{\bfw'}) \geq\ & \min\{w_1, w'_1\}T(P_1, Q_1) + \min\{w_2, w'_2\}T(P_2, Q_2)\\
        & + (w_2' - \min\{w_2, w'_2\})T(P_1, Q_2) + (w'_1 - \min\{w_1, w'_1\})T(P_2, Q_1).
    \end{align*}
     In addition, similar results hold for any $m \geq 2$. 
\end{proposition}


\subsection{Characterization of \texorpdfstring{$T((X|I,I),(Y|I,I))$}{} and proof of Proposition \ref{prop:symmetry}}

The type I error $\alpha_i(t)$ can be represented by the trade-off function $T(P_i,Q_i)$ and the rejection region decided by $t$ and $c$.
For simplicity, in the following of this section, we only discuss the continuous case with $\mathbb{P}[q/p = t] = 0$, where the trade-off function is differentiable, and rewrite $\alpha_i(t,c) = \alpha_i(t).$
Precisely, we have the following proposition.
\begin{proposition}\label{prop:para-represent}
     Let $T(P_i,Q_i)= f_i$ with some differentiable trade-off function $f_i$. Suppose that 
    \begin{align}
    \label{eq:finite-int-delta}    
    \int_{\mathbb{R}} \delta\left(t - \frac{q_i(x)}{p_i(x)}\right) q_i(x) dx = \int_{\frac{q_i(x)}{p_i(x)}=t} \frac{q_i(x)}{D\left(\frac{q_i}{p_i}\right)(x)} dx  <\infty,
    \end{align}
    for any  $t>0$ and $1\leq i \leq m$,
    where $\delta$ is the Dirac delta function and $D(\frac{q_i}{p_i})(x)$ is the weak derivative of $q_i/p_i.$
    Then, we have     
    \begin{align}
 \label{eq:ReleaseI-Finite}
        \alpha(t) = \sum_{i=1}^m w_i\left(f_i'\right)^{-1}(-t) \quad \hbox{and} \quad T(P_I,Q_I)(\alpha(t)) = \sum_{i=1}^m w_i f_i\left(\left(f_i'\right)^{-1}(-t)\right).
    \end{align}
\end{proposition}

The key observation from Proposition \ref{prop:para-represent} is that 
\begin{align*}
    \frac{d f_i}{d\alpha_i}\Big|_{\alpha_i = \alpha_i(t)} = -t
    \quad
\hbox{and}\quad 
\frac{dT(X|I,Y|I)(\alpha)}{d\alpha}\Big|_{\alpha = \alpha(t)}  = -t,
\end{align*}
which means $f_i$ and $T(P_I,Q_I)$ have the same derivative at points $\alpha_i(t)$ and $\alpha(t)$ induced by the same threshold $t.$
This observation is important in our analysis of the applications to shuffling models and to derive the joint convexity of $F$-divergences.
Equation \eqref{eq:finite-int-delta} is not a strong assumption. For example, when $q_i/p_i$ is strictly monotone, we have $\delta(t - q_i/p_i(x)) = \delta(x - x_t)$ with $q_i/p_i(x_t) = t$ and
\begin{align*}
     \int_{\mathbb{R}} \delta\left(t - \frac{q_i(x)}{p_i(x)}\right) q_i(x) dx  = q_i\left( x_t\right) \leq 1.
\end{align*} 

\begin{corollary}
     Let $T(P_i,Q_i)= f_i$ with some differentiable trade-off function $f_i$. Suppose that $p_i/q_i$ is monotone.
    Then, we have $
        \alpha(t) = \sum_{i=1}^m w_i\left(f_i'\right)^{-1}(-t)$ and  $T(P_I,Q_I)(\alpha(t)) = \sum_{i=1}^m w_i f_i\left(\left(f_i'\right)^{-1}(-t)\right).$
\end{corollary}

\begin{proof}[Proof of Proposition \ref{prop:para-represent}]

Rewrite $\alpha_i(t)$ and $\beta_i(t) = f_i(\alpha_i(t))$ as
\begin{align*}
    \alpha_i(t) = \int_{\mathbb{R}} \mathds{1}\left[q_i(x)/p_i(x)<t\right]p_i(x)dx
    =\int_{\mathbb{R}} \mathds{1}\left[q_i(x)/p_i(x)<t\right]\frac{p_i(x)}{q_i(x)} q_i(x)dx
\end{align*}
and 
\begin{align*}
    \beta_i(t) = \int_{\mathbb{R}} \mathds{1}\left[p_i(x)/q_i(x)>t\right]q_i(x)dt.
\end{align*}
Then we have
\begin{align*}
    \frac{d\alpha_i}{dt} = \int_{\mathbb{R}} -\delta\left(t - \frac{q_i}{p_i}(x) \right)\frac{p_i}{q_i}(x) q_i(x) dx,
\end{align*}
where $\delta$ is the Dirac delta function.
Note that 
\begin{align*}
  \delta\left(t - \frac{q_i}{p_i}(x) \right)\frac{p_i}{q_i}(x) = \frac{\delta\left(t - \frac{q_i}{p_i}(x) \right)}{t}.
\end{align*}
We get
\begin{align*}
 \frac{d\alpha_i}{dt} = -\frac{1}{t}\int_{\mathbb{R}} \delta\left(t - \frac{q_i}{p_i}(x) \right) q_i(x) dx.
\end{align*}
Since
\begin{align*}
    \frac{d\beta_i}{dt} = \int_{\mathbb{R}} \delta\left(t - \frac{q_i}{p_i}(x) \right) q_i(x) dx,
\end{align*}
we obtain
\begin{align*}
   f_i'(\alpha_i) = \frac{d \beta_i}{d\alpha_i} = \frac{d\beta_i}{dt} \frac{dt}{d\alpha_i} = -t
\end{align*}
and $\alpha_i = (f_i')^{-1}(-t).$
\end{proof}

Based on the proof of Proposition \ref{prop:para-represent}, we now prove Proposition \ref{prop:symmetry}.

\begin{proof}[Proof of Proposition \ref{prop:symmetry}]
    It is sufficient to show that, for any $\left(\alpha(t),\beta(t)\right)$ such that $\beta(t) = f(\alpha(t))$, there is a $\widetilde{t}$ such that $\beta(\widetilde{t}) = \alpha(t) = f(\alpha\left(\widetilde{t}\right))$ with 
    \begin{align*}
        \alpha(\widetilde{t}) = \sum_{i=1}^m w_i \alpha_i(\widetilde{t}) = \sum_{i=1}^m w_i \beta_i(t). 
    \end{align*}
As $f_i$ is symmetric, for any $1\leq i \leq m$, there is $\widetilde{t}_i$ such that $\alpha_i(\widetilde{t}_i) = \beta_i(t)$ and $\beta_i(\widetilde{t}_i) = \alpha_i(t).$
So, it is enough to show that $\widetilde{t}_i = \widetilde{t}$ for all $i$.
Recall that in the proof of Proposition \ref{prop:para-represent}, we have
\begin{align*}
    \frac{d\beta_i}{d\alpha_i}\Bigg|_{\alpha_i = \alpha_i(t) } = -t.
\end{align*}
Now we consider $d\beta_i(\widetilde{t}_i)/d\alpha_i(\widetilde{t}_i)=-\widetilde{t}_i.$
On the other hand, since
\begin{align*}
\beta_i(\widetilde{t}_i) = f_i(\alpha_i(\widetilde{t_i}))=f_i^{-1}(\alpha_i(\widetilde{t_i})), 
\end{align*}
using the inverse funciton theorem, we have
\begin{align*}   d\beta_i(\widetilde{t}_i)/d\alpha_i(\widetilde{t}_i) = \frac{1}{f_i'(f_i^{-1}(\alpha_i(\widetilde{t_i})))} = \frac{1}{f_i'(\alpha_i(t))} = -\frac{1}{t},
\end{align*}
where the second equality is because $\alpha_i(\widetilde{t}_i) = \beta_i(t) = f_i(\alpha_i(t))$ and the third equality is from $\alpha_i(t) = (f_i')^{-1}(-t)$ that appears in Proposition \ref{prop:para-represent}.
Overall, we obtain $\widetilde{t}_i = \frac{1}{t}$ for all $i.$
\end{proof}

\subsection{Conversion from Lemma \ref{lemma:inequality} to \texorpdfstring{$F$}{}-divergences}

In this section, we investigate the relationship between Lemma \ref{lemma:inequality} and $F$-divergences.
For two distributions $P$ and $Q$, the $F$-divergence between $P$ and $Q$ is given by
\begin{align*}
    D_F(P\|Q) = \int_{pq>0}F(p/q) dQ + F(0)\mathbb{P}_Q[p=0] + \tau_F \cdot \mathbb{P}_P[q=0],
\end{align*}
where $F(0) =\lim_{s\rightarrow 0 } F(s)/s$ and $\tau_F =\lim_{t\rightarrow \infty } F(t)/t. $

To convert Lemma \ref{lemma:inequality} to $F$-divergences, we recall the relationship between trade-off functions and $F$-divergences in \cite{dong2019gaussian}.
Precisely, for any $F$-divergence $D_F(P\|Q)$ between two distributions $P$ and $Q$, there is a functional $l_F(T(P,Q))$ such that $D_{F}(P\|Q) = l_F(T(P,Q))$.
This functional $l_F$ can be specified using the following lemma.

\begin{lemma}[Proposition B.4 in \cite{dong2019gaussian}]
Let $z_f:= \inf\{x \in [0,1], f(x)=0\}$ be the first zero of a trade-off function $f$. The functional $l_F$ that computes the $F$-divergence $D_F$ has the following expression
\begin{align*}
    l_F(f) = \int_0^{z_f}F\left(\frac{1}{|f'(x)|}\right) \cdot |f'(x)| dx + F(0) (1 - f(0)) + \tau_F(1 - z_f).
\end{align*}
In particular, if $f$ is symmetric with $f(0)=1$, then we have
\begin{align*}
    l_F(f) = \int_0^{z_f}F\left(\frac{1}{|f'(x)|}\right) \cdot |f'(x)| dx.
\end{align*}
\end{lemma}

Now we prove Proposition \ref{prop:F-divergence}.

\begin{proof}[Proof of Proposition \ref{prop:F-divergence}.]
Let $f(\alpha) := T(P_I, Q_I)(\alpha).$
According to Lemma \ref{lemma:inequality} and Proposition \ref{prop:para-represent}, we have
\begin{align*}
    f(\alpha(t)) = \sum_{i=1}^m w_i f_i(\alpha_i(t))
\end{align*}
with $f_i = T(P_i,Q_i)$ and $\alpha_i(t) = (f_i')^{-1}(t).$
Note that 
\begin{align*}
    \frac{d{f}}{d\alpha} = \frac{d{f}}{dt} \frac{dt}{d\alpha} = \frac{\sum_{i=1}^m w_i f_i'(\alpha_i(t)) \alpha_i'(t)}{\sum_{i=1}w_i\alpha'_i(t)}=-t,
\end{align*}
where the second equality is a result of the inverse function theorem and the last equality is because $f_i'(\alpha_i(t)) = -t$.
Thus, using Lemma \ref{lemma:inequality} and Lemma B.2 in \cite{dong2019gaussian}, we have
\begin{align*}  D_{F}\left(P_\bfw\big\| Q_\bfw\right) &= l_F(T(P_\bfw,Q_\bfw))\leq l_F({f})=\int_0^1 F\left(\frac{1}{|f'(\alpha)|}\right)|f'(\alpha)|d\alpha
    \\
    &=\int_0^\infty F\left(\frac{1}{t}\right) \cdot t\cdot \frac{d\alpha}{dt} dt
    =
    \sum_{i=1}^m w_i \int_0^\infty F\left(\frac{1}{t}\right) \cdot t\cdot\alpha_i'(t) dt.
\end{align*}

Since
\begin{align*}
  \frac{df_i}{d\alpha_i} = \frac{df_i}{dt} \frac{dt}{d\alpha_i} = -t,
\end{align*}
we have
\begin{align*}
   \int_0^\infty F\left(\frac{1}{t}\right) \cdot t\cdot\alpha_i'(t) dt = \int_0^1 F\left(\frac{1}{|f_i'(\alpha_i)|}\right) \cdot |f_i'(\alpha_i)|\cdot d\alpha_i = D_F(P_i\|Q_i)
\end{align*}
and
\begin{align*}  
D_{F}\left(P_\bfw\big\| Q_\bfw\right) \leq  \sum_{i=1}^m w_i \int_0^1 F\left(\frac{1}{|f_i'(\alpha_i|)}\right) \cdot |f_i'(\alpha_i)|\cdot d\alpha_i =\sum_{i=1}^m w_i D_F(P_i\|Q_i).
\end{align*}
\end{proof}

Let $H_{\gamma}(P\|Q) = \int \left[\frac{p(x)}{q(x)} - \gamma\right]_+dQ(x)$ be the hockey-stick divergence between $P$ and $Q$.
Note that the hockey-stick divergence $H_{\gamma}$ is an $F$-divergence with $F(s) = (s - \gamma)_+$. It holds the following Corollary.

\begin{corollary}[An application of Proposition \ref{prop:F-divergence} to the hockey-stick divergence.]
\label{coro:main-thm-to-DP}
For any $\gamma\geq 1$, we have
\begin{align*}
    H_\gamma(P_{\bfw} \| Q_{\bfw}) &\leq \sum_{i=1}^m w_i H_\gamma(P_i \| Q_i).
\end{align*}
\end{corollary}

Let $R_{\widetilde{\alpha}}(P\| Q)$ be the R\'enyi divergence of order $\widetilde{\alpha}$ between two distributions $P$ and $Q$.
$R_{\widetilde{\alpha}}$ is not an $F$-divergence with convex $F$.
However, the scaled exponentiation of  R\'enyi divergence $e^{(\widetilde{\alpha} - 1) R_{\widetilde{\alpha}}}$ is known as the power divergence that corresponds to $F_{\widetilde{\alpha}}(s) = \frac{s^{\widetilde{\alpha}} - \widetilde{\alpha} (s-1) -1}{\widetilde{\alpha} (\widetilde{\alpha}-1)}.$
The joint convexity of the scaled exponentiation of  R\'enyi divergence can be derived from Proposition \ref{prop:F-divergence}.
\begin{corollary}[An application of Proposition \ref{prop:F-divergence} to the R\'enyi divergence]
\label{coro:main-thm-to-RDP}
    It holds
\begin{align*}
    e^{(\widetilde{\alpha} - 1) R_{\widetilde{\alpha}}\left(P_{\bfw} \big\| Q_{\bfw}\right)} \leq \sum_{i=1}^m w_i \int_0^1 \left| f_i'(x) \right|^{1 - \widetilde{\alpha}} dx = \sum_{i=1}^m w_i e^{(\widetilde{\alpha} - 1) R_{\widetilde{\alpha}}\left(P_i \big\| Q_i\right)},
\end{align*}
for any $\widetilde{\alpha} \geq 1.$
\end{corollary}
 Corollary \ref{coro:main-thm-to-RDP} is in line with Lemma 4.1 in \cite{ye2022differentially}. Conversion from a trade-off function to an $F$-divergence is straightforward using Section B in \cite{dong2019gaussian}. However, conversion from an $F$-divergence to a trade-off function is highly non-trivial. In fact, the trade-off function is a (global) integral over the whole space while Lemma \ref{lemma:inequality} holds pointwisely, which is a local property. 
This explains why the divergence-based DP is not as informative as $f$-DP since some information is lost due to the integration.
Specifically, the following proposition says the length of a trade-off function (which is also a global property) is related to an $F$-divergence.

\begin{proposition}
    Let $f$ be a trade-off function and let $\mathrm{len}(f)$ be the length of $f$. Then we have
    \begin{align*}
        \mathrm{len}(f) = \int_{0}^1 \sqrt{(1 + f'(x)^2)} = \int_{0}^1 F\left(\frac{1}{|f_i'(x)|} \right) |f_i'(x)|dx
    \end{align*}
with $F(y)=\sqrt{1 + y^2}. $
\end{proposition}

\section{Technical details for shuffling models in Section \ref{sec:shuffling}}    
In this section, we delve into the discussion on shuffling models as introduced in Section \ref{sec:shuffling}. To specify the distribution of the output of shuffling models, we refer to the work of \cite{feldman2023stronger}.
Recall that in shuffling models, a dataset $\mathcal{D}\in\mathcal{Z}^n$ is privatized by a local randomizer $\mathcal{A}:\mathcal{Z}^n \rightarrow \widetilde{\mathcal{Z}}^n$ that is $\epsilon_0$-DP.
Then, a shuffler $\mathcal{A}_{\mathrm{Shuffle}}:\widetilde{\mathcal{Z}}^n \rightarrow \widetilde{\mathcal{Z}}^n$ applies a uniformly random permutation to $\mathcal{A}(D).$

According to Theorem 3.1 in \cite{feldman2023stronger}, for any two neighboring datasets $\mathcal{D}_0$ and $\mathcal{D}_1$, there is a post-processing procedure $\mathrm{Proc}$ such that $\mathcal{A}_{\mathrm{Shuffle}}\circ \mathcal{A}(\mathcal{D}_0) = \mathrm{Proc}(X)$ and $\mathcal{A}_{\mathrm{Shuffle}}\circ \mathcal{A}(\mathcal{D}_1) = \mathrm{Proc}(Y)$ with $X\sim P$ and $Y\sim Q$.
Here $P = (1 - w)P_0 + wQ_0$ and $Q = (1 - w)Q_0 + wP_0$ are two distributions with $(A+1, C - A)\sim P_0$, $(A, C - A +1)\sim Q_0$, and $w=1/(e^{\epsilon_0}+1),$ where $A\sim \mathrm{Binom}(C,1/2)$ and $C\sim \mathrm{Binom}(n-1, 2/(e^{\epsilon_0} + 1)).$
As a result of Lemma \ref{lemma:processing}, it is sufficient to bound $T(P,Q).$

In the subsequent part of this section, we bound $T(P,Q)$ through the following two steps.
First, we bound $T(P_0,Q_0)$ using the joint concavity Lemma \ref{lemma:inequality}, and the resulting bound is provided by Proposition \ref{prop:shuffling}.
Next, we can establish a bound for the trade-off function $T(P,Q)$ by applying Proposition \ref{prop:advancedShuffling}.

\subsection{Proof of Proposition \ref{prop:shuffling} and Theorem \ref{thm:shuffling}}
In this section, we present the proof of Theorem \ref{thm:shuffling}. Since the proof relies on Proposition \ref{prop:shuffling}, we will begin by proving Proposition \ref{prop:shuffling}.

\subsubsection{Proof of Proposition \ref{prop:shuffling}}
\label{sec:proof_prop_shuffling}
The upper bound can be derived directly from Lemma \ref{lemma:inequality}.
It is enough to show that the equality in Proposition \ref{prop:shuffling} holds.

Let $p_0$ and $q_0$ be the probability density functions of $P_0$ and $Q_0$, respectively.
As stated in Lemma \ref{lemma:NP}, our initial step is to verify the likelihood ratio $p_0/q_0$.
For any $(a,b)$ belongs to the support of both $P_0$ and $Q_0$, we have
\begin{align*}  p_0(a,b)=\mathbb{P}\left[A+1=a, C-A = b \right] &= \mathbb{P}\left[\left.A_i=a-1, i-A_i = b\right| C=i \right] \mathbb{P}[C=i]
\\
&=\mathbb{P}\left[\left.A_i=a-1, i =a + b-1\right| C=i \right] \mathbb{P}[C=i]
\end{align*}
and
\begin{align*}  q_0(a,b)=\mathbb{P}\left[A=a, C-A+1 = b \right] = \mathbb{P}\left[\left.A_i=a, i =a +  b-1\right| C=i \right]\mathbb{P}[C=i].
\end{align*}
Thus,
\begin{align*}  \frac{p_0(a,b)}{q_0(a,b)} = \frac{\mathbb{P}\left[A_i=a-1, i =a + b-1\right]}{\mathbb{P}\left[A_i=a-1, i =a + b-1\right]}=\frac{a}{b}.
\end{align*}

When $p_0(a,b) = 0 $ and $q_0(a,b)\neq 0$, we have $a= 0$, $b = C+1$, and $a/b=0 =p_0(a,b)/q_0(a,b).$
Similarly, for the case $p_0(a,b) \neq 0 $ and $q_0(a,b)= 0$, we have $a= C+1$, $b = 0$, and $a/b=\infty =p_0(a,b)/q_0(a,b).$
In conclusion, it holds $\frac{p_0(a,b)}{q_0(a,b)} = \frac{a}{b}$.

The corresponding type I error is \begin{align*}
        \alpha_0(t,c) = \mathbb{P}_{A,C}\left[ \frac{A+1}{C-A} < t\right] + c \mathbb{P}_{A,C}\left[ \frac{A+1}{C-A} = t\right],
\end{align*}
for any $t\geq 0$ and $c\in[0,1]$, and the type II error is given by
\begin{align*}
        \beta_0(t,c) = \mathbb{P}_{A,C}\left[ \frac{A}{C-A + 1} > t\right] + (1 - c) \mathbb{P}_{A,C}\left[ \frac{A}{C-A + 1} = t\right].
 \end{align*}
 Since the distributions of $A$ and $C$ are discrete, the trade-off function between $\alpha_0$ and $\beta_0$ is piece-wise linear and each knot corresponds to some $t$ with $c=1$.
 For simplicity, we define
 \begin{align*}
     \alpha_0(t) =   \alpha_0(t,1) = \mathbb{P}_{A,C}\left[ \frac{A+1}{C-A} \leq t\right]
 \end{align*}
 and
\begin{align*}
       \beta_0(t) =  \beta_0(t,1) = \mathbb{P}_{A,C}\left[ \frac{A}{C-A + 1} > t\right].
 \end{align*}

Note that given $C=i$, $(A+1,C-A) = (A_i + 1, i-A_i)$ with $A_i\sim \mathrm{Binom}(i,1/2)$. We have $P_0$ is a mixture of $\{(A_i + 1, i-A_i)\}_{i=0}^{n-1}$ and the weights are $\{w_i^0\}_{i=0}^{n-1}$ with $w_i^0 = \mathbb{P}[C = i].$
Using this observation, we rewrite
\begin{align*}
    \alpha_0(t) = \mathbb{E}_C\left\{ \mathbb{P}\left[ \left.\frac{A +1}{C - A } \leq t\right| C\right] \right\}
    =\sum_{i=0}^{n-1} w_i^0 \left\{ \mathbb{P}\left[\frac{A_i+1}{i - A_i } \leq t\right] \right\} =: \sum_{i=0}^{n-1} w_i^0 \alpha_i^0(t).
\end{align*}
For each $\alpha_i^0(t)$, it holds
\begin{align*}
    \alpha_{i}^0(t) = \mathbb{P}\left[\frac{A_i+1}{i - A_i } \leq t\right] = F_{i}\left(i - \frac{i+1}{t+1}\right).
\end{align*}
Similarly, we can decompose $\beta_0(t) = \sum_{i=0}^{n-1} w_i^0 \beta_i^0(t)$, where
\begin{align*}
    \beta_{i}^0(t) = \mathbb{P}\left[\frac{A_i}{i - A_i+1 } > t\right]  = 1 - F_i\left(i+1 - \frac{i+1}{t+1}\right).
\end{align*}
Since
\begin{align*}
    \alpha_i^0(t) = F_i\left(i- \frac{i+1}{t+1}\right) = F_i(s_i)
\end{align*}
with $s_i = F_i^{-1}(\alpha_i^0(t))\in \mathbb{N}$, we have
\begin{align*}
  \beta_i^0(t) = F_i\left(i+1- \frac{i+1}{t+1}\right) = F_i(s_i + 1)  =F_i(F_i^{-1}(\alpha_i^0(t)) + 1),
\end{align*}
where the second equality is because the support of $F_i$ is $\{0,1,\cdots,i\}$ and $s_i,s_i+1 \in\mathbb{N}.$
In conclusion, it holds
\begin{align*}
    \beta_0(t) = \sum_{i=0}^{n-1}w_i^0 \beta_i^0(t) = \sum_{i=0}^{n-1}w_i^0\left\{  F_i(F_i^{-1}(\alpha_i^0(t)) + 1) \right\},
\end{align*}
which completes the proof.


\subsubsection{Proof of Theorem \ref{thm:shuffling}}

Now we compute the trade-off function $f_{\text{Shuffle}}$ at each knot $\alpha(t)$.
Let $f_0 = \mathcal{C}(T(P_0,Q_0))$ be the symmetrization of $T(P_0,Q_0)$ and rewrite $f_{P,Q} = T(P,Q).$

\begin{proof}[Proof of Theorem \ref{thm:shuffling}]
The proof is a straightforward conclusion from Proposition \ref{prop:advancedShuffling}.
To complete the proof, we still need to show that $\mathcal{C}(f_{P,Q}) = \mathcal{C}(f_{\mathrm{Shuffle}})$ with $f_{\mathrm{Shuffle}} = 2w\cdot \mathrm{Id} + (1 - 2w)\cdot T(P_0,Q_0).$
By the proof of Proposition F.2 in \cite{dong2019gaussian}, we have $\mathcal{C}(f)(x) = f(x)$ for any $x\leq \bar{x}_f$ and any trade-off function $f$, where $\bar{x}_f$ is such that $\inf\{x\in[0,1], -1\in f(x)\}.$
Thus, we have $f_0(x) = T(P_0,Q_0)(x)$ for any $x\leq \bar{x}_{f_0}.$
Note that $-1 \in \partial f_{P,Q}(x)$ if and only if $-1 \in \partial f_0(x)$.
We obtain $\bar{x}_{f_{P,Q}} = \bar{x}_{f_0}$. Moreover, for $x\leq \bar{x}_{f_0},$ it holds
\begin{align*}
    f_{P,Q}(x) = 2w \mathrm{Id}(x) + (1 -2w) f_0(x) = 2w \mathrm{Id}(x) + (1 -2w) T(P_0,Q_0)(x) = f_{\mathrm{Shuffle}}(x).
\end{align*}
Using the symmetry of $\mathcal{C}(f_{\mathrm{Shuffle}})$ and $\mathcal{C}(f_{P,Q})$ in Equation \eqref{eq:double-conjugate}, we have $\mathcal{C}(f_{\mathrm{Shuffle}}) = \mathcal{C}(f_{P,Q}).$
\end{proof}

\subsection{Proof of Corollary \ref{coro:shuffling f to epsilon}}
According to \cite{dong2019gaussian} and the proof of Theorem \ref{thm:shuffling}, we have $\mathcal{A}_{\mathrm{Shuffle}}\circ\mathcal{A}$ is $(\epsilon,\delta)$-DP with
$\delta(\epsilon) = 1 + f_{P,Q}^{*} (-e^{\epsilon})$.

Recall the definition $f^*(y) = \sup_{\alpha}\{ y\alpha - f(\alpha) =: h_y(\alpha)\}.$
Then, by the first-order optimality condition, we have $f^*(y) = h_y(\widetilde{\alpha})$ with $\widetilde{\alpha} = \inf\{\alpha\in[0,1], 0\in \partial h_y({\alpha})\}$.

For $\alpha(t)$ such that $f_{P,Q}$ is differentiable at $\alpha(t)$, we have $f_{P,Q}'(\alpha(t)) = -2w + (1 - 2w) l(t)$
with
\begin{align*}
      l\left(t\right) =-\frac{\sum_{i=0}^{n-1} w_i^0 p_{i}\left(\left\lfloor i +1  - \frac{i+1}{t+1}  \right\rfloor \right)}{\sum_{i=0}^{n-1} w_i^0 p_{i}\left(\left\lfloor i  - \frac{i+1}{t+1}\right\rfloor \right)}.
\end{align*}
Here $p_i$ is the probability mass function of $A_i.$
Thus, $h_y'(\alpha(t)) = y + 2w - (1 - 2w) l(t).$
And $\widetilde{\alpha}(t)$ is then given by $\inf\{\alpha: h_y'(\alpha) \leq 0\}.$
Since $\alpha(t) $ is an increasing function of $t$, we obtain
\begin{align*}
    \widetilde{\alpha}(t) = \alpha(t_y), \qquad \hbox{ with } t_y = \inf\{t: y + 2w - (1 -2w) l(t) \leq 0\}
\end{align*}
and
\begin{align*}
    f_{\mathrm{Shuffle}}^*(y) = h_y(\alpha(t_y)) = y\alpha(t_y) - f_{P,Q}(\alpha(t_y)) = y\alpha(t_y) - \beta(t_y).
\end{align*}
We end the proof by taking $y = -e^{\epsilon}.$

\section{Omited details of Section \ref{sec:DPGD}}

\subsection{Proof of Theorem \ref{thm:DPGD}}

According to a continuous version of Lemma \ref{lemma:inequality} , that is given by Equation (\ref{eq:type_I_cts_inequality}) and Equation (\ref{eq:type_II_cts_inequality}), $T(X,Y)$ in Theorem \ref{thm:DPGD} is lower bounded by the trade-off function $T(P_I,Q_I)$ with $I\sim\mathcal{N}(0,1)$, $P_I=\mathcal{N}(0,1)$ and $Q_I=\mathcal{N}(\mu_I,1).$
 For this example,we have $p_I(x) = e^{-x^2/2}$ and $q_I(x)  = e^{-(x-\mu_I)^2/2}$.
 Then the type I error is 
 \begin{align*}
      \alpha(t) &= \int_{-\infty}^\infty \mathbb{P}_{X\sim\mathcal{N}(0,1)}\left[-\mu_w X + \frac{\mu_w^2}{2}\leq t\right] e^{-\frac{w^2}{2}}dw
      \\
&=\int_{\mu_w\leq 0}\Phi(t_w)e^{-\frac{w^2}{2}}dw + \int_{\mu_w>0}\Phi(-t_w) e^{-\frac{w^2}{2}}dw 
 \end{align*}
 with $t_w = -\frac{t}{\mu_w}            +\frac{\mu_w}{2}.$
Similarly, the type II error is 
\begin{align*}
    \beta(t) &=\int_{-\infty}^\infty \mathbb{P}_{X\sim\mathcal{N}(\mu_w,1)}\left[-\mu_w X + \frac{\mu_w^2}{2}> t\right] e^{-\frac{w^2}{2}}dw
    \\
    &=\int_{\mu_w\leq 0}\left[ \Phi(-t_w + \mu_w)\right]e^{-\frac{w^2}{2}}dw + \int_{\mu_w>0}\left[ \Phi(t_w - \mu_w)\right] e^{-\frac{w^2}{2}}dw,
\end{align*}
which completes the proof of Theorem \ref{thm:DPGD}.

\subsection{Examples for different loss functions}
Recall the noiseless linear model with $\mathcal{D}_0=\{(x_i,y_i)\}_{i=1}^n$ 
with $y_i=a x_i$ and $x_i^2 = 1$ for some constant $a$ and we defined $\mathcal{D}_1$ by removing an arbitrary element in $\mathcal{D}_0$.

\begin{example}[Least-squares loss without gradient clipping]
\label{exam:DPGD}
For linear least squares regression with $\ell(\theta,\mathcal{D}) = \sum_{i=1}^n(y_i - \theta x_i)^2$ 
, we have $g(\theta,\mathcal{D}) = \sum_{i=1}^n(y_i - \theta x_i) (-x_i)$ and $\mu_I =a - I$.
\end{example}
In Example \ref{exam:DPGD}, the gradient is unbounded due to an unbounded initializtion, and so is its sensitivity.
In this example, the dominate pair for $\theta(\mathcal{D}_0)$ and $\theta(\mathcal{D}_1)$ is $(\mathcal{N}(0,\sigma^2), \mathcal{N}(0,\sigma^2) + a - I).$
Note that $\mathcal{N}(0,\sigma^2) + a - I$ is a Gaussian distribution with mean $a$ and variance $1 + \sigma^2.$
Thus, under the framework of RDP, the goal is to bound the R\'enyi divergence between two Gaussian distributions with different variances, which is unbounded for $\widetilde{\alpha}$ large enough.

 \begin{example}[Least squares loss with gradient clipping]
\label{exam:DPGD_clip}
Consider a linear least squares regression problem in Example \ref{exam:DPGD}. For DP-GD with gradient clipping, we have
\begin{align*}
    \mu_{I} = \left\{
     \begin{array}{cc}
        a - I,  & |a - I| \leq c ,  \\
       c,   & a - I \geq c, \\
       -c, & a - I\leq -c.
     \end{array}
    \right.
\end{align*}
\end{example} 

 \begin{example}[Logistic loss]
For the logistic loss, we have
\begin{align*}
    \mu_I^{\max}= \sup_{x,y} \left|\frac{e^{-I\cdot yx}}{1 + e^{- I \cdot yx}} \right|
\end{align*}
as the gradient of the logistic loss is the softmax function.
$\mu_I^{\max}$ is bounded when $|xy| \leq M$ for some $M>0$.
Furthermore, extending the logistic loss to other strongly convex losses is straightforward, given that the key feature is the gradient being a monotone function of $I \cdot yx$. 
\end{example}

\section{Technical details of Lemma \ref{lemma: f advanced joint convexity} and corresponding conclusions}
In this section, we discuss the omitted details of Section \ref{sec:theory} including the proofs of the advanced joint concavity (Lemma \ref{lemma: f advanced joint convexity}).

\subsection{Proof of Lemma \ref{lemma: HS advanced joint convexity}, Lemma \ref{lemma: f advanced joint convexity}, and corresponding results}

In this section, we establish the proof of Lemma \ref{lemma: HS advanced joint convexity} and Lemma \ref{lemma: f advanced joint convexity}. Before delving into the proof, we revisit Proposition \ref{prop:advancedShuffling} that directly stem from the application of Lemma \ref{lemma: HS advanced joint convexity} and Lemma \ref{lemma: f advanced joint convexity}.
The proof of Proposition \ref{prop:advancedShuffling} is included as part of the proof of Theorem \ref{thm:shuffling} in Section \ref{sec:proof_prop_advanced_shuffling}.
Similar to the proof of Proposition \ref{prop:advancedShuffling}, for $P = (1- w) P_0 + w P_1$ and $Q = (1 - w) P_0 + w Q_1$ that appear in the analysis sub-sampling \cite{balle2018privacy}, we have the following proposition.

\begin{proposition}
    For $P = (1- w) P_0 + w P_1$ and $Q = (1 - w) P_0 + w Q_1$, we have
    \begin{align*}
        T(P,Q) \geq \mathcal{C}\left( (1 -w) \mathrm{Id} + w T(P_1,Q_1) \right).
    \end{align*}
\end{proposition}


\begin{proof}[Proof of Lemma \ref{lemma: HS advanced joint convexity}]
We first invoke an important equality from \cite{balle2018privacy}.
    \begin{align} \label{eqn: hockey-stick equality}
        H_{\alpha} (P \| Q) := \sup_{E}\left\{ P(E) - \alpha Q(E)\right\} = \int \left[ p(z) - \alpha q(z) \right]_{+} dz.
    \end{align}
    According to \cite{barthe2013beyond}, we have $\mathcal{A}$ is $(\epsilon, \delta)$-differentially private if and only if $H_{e^{\epsilon}}(\mathcal{A}(D_0) \| \mathcal{A}(D_1)) \leq \delta$ for every neighboring $D_0$ and $D_1$. 
    We now recall the following two equations which are constraints on $\epsilon_0,\epsilon_1, \gamma,$ and $\eta$:
    \begin{align}
        \label{eqn: joint epsilon}
        e^{\epsilon'} = (1 - w)e^{\epsilon_0} + w e^{\epsilon_1}
    \end{align}
    and 
    \begin{align}
        \label{eqn: gamma and beta}
        w e^{\epsilon'} = (1 - w)\gamma e^{\epsilon_0} + w \eta e^{\epsilon_1}.
    \end{align}
    It is evident from Equations \eqref{eqn: joint epsilon} and \eqref{eqn: gamma and beta} that 
    \begin{align*}
        \exp(\epsilon_0) (1 - w) \gamma + \exp(\epsilon_1) w \eta = \exp(\epsilon') w,
    \end{align*}
    and 
    \begin{align*}
        \exp(\epsilon_0) (1 - w) (1 - \gamma) + \exp(\epsilon_1) w (1 - \eta) =\ & \exp(\epsilon') (1 - w).
    \end{align*}
    Thus, we have
    \begin{align*}
        &(1 - w) P_{1} + w P_{2} - e^{\epsilon'} \left( (1-w) {Q_1} + w {Q_2} \right)\\
        =\ & (1 - w) \left(P_{1} - e^{\epsilon_0} (1 - \gamma) {Q_1} - e^{\epsilon_0} \gamma {Q_2} \right) + w \left(P_2 - e^{\epsilon_1} (1 - \eta) {Q_1} - e^{\epsilon_1} \eta Q_2 \right).
    \end{align*}
    This completes the proof of proposition by equation \eqref{eqn: hockey-stick equality}. 
\end{proof}

\begin{proof}[Proof of Lemma \ref{lemma: f advanced joint convexity}]
 According to Lemma \ref{lemma: HS advanced joint convexity} and Proposition 2.12 in \cite{dong2019gaussian}, we aim to find a trade-off function $F$ such that
    \begin{align} \label{eqn: trade off function}
    \begin{split}
        F^{*}(- e^{\epsilon'}) =\ & 
        (1 - w)(1 - \gamma) f_{0,0}^*(- e^{\epsilon_0}) + w(1 - \eta) f_{1,0}^* \left(- e^{\epsilon_1} \right)\\
        &+ (1 - w) \gamma f_{0,1}^{*} (- e^{\epsilon_0}) + w \eta f_{1,1}^{*} \left(- e^{\epsilon_1} \right), 
    \end{split}
    \end{align}
    where $\epsilon_0, \epsilon_1, \gamma, \eta$ satisfy equations \eqref{eqn: joint epsilon} and \eqref{eqn: gamma and beta}. 
Let $y = -e^{\epsilon'} < -1$ and, for fixed $\gamma < w < \eta$, Equations \eqref{eqn: joint epsilon} and \eqref{eqn: gamma and beta} imply that 
\begin{align*}
    - e^{\epsilon_0} = \frac{y (\eta - w)}{(1 - w) (\eta - \gamma)}, \quad \hbox{ and } \quad     - e^{\epsilon_1} = \frac{(w - \gamma) y}{w (\eta - \gamma)}.
\end{align*}
Therefore, $F(x)$ is given by the double conjugate: 
\begin{align*}
    F(x) =\ & \sup_{y} xy - F^*(y)\\
    =\ & \sup_{y} xy - (1 - w)(1 - \gamma) f_{0,0}^*\left(\frac{y (\eta - w)}{(1 - w) (\eta - \gamma)} \right) - w(1 - \eta) f_{1,0}^* \left(\frac{(w - \gamma) y}{w (\eta - \gamma)}\right)\\
    &- (1 - w) \gamma f_{0,1}^{*} \left(\frac{y (\eta - w)}{(1 - w) (\eta - \gamma)} \right) - w \eta f_{1,1}^{*} \left(\frac{(w - \gamma) y}{w (\eta - \gamma)}\right).
\end{align*}
For $i = 0,1$, define
\begin{align*}
    F_{0, i}(x) = f_{0,i} \left(\frac{x (1 - w) (\eta - \gamma)}{(\eta - w)} \right),\quad \hbox{and} \quad
    F_{1, i}(x) := f_{1,i} \left(\frac{ x w (\eta - \gamma)}{(w - \gamma)}\right).
\end{align*}
Thus, we have 
\begin{align*}
    F_{0, i}^*(y) = f_{0,i}^*\left(\frac{y (\eta - w)}{(1 - w) (\eta - \gamma)} \right),\quad \hbox{and} \quad
    F_{1, i}^*(y) = f_{1,i}^{*} \left(\frac{(w - \gamma) y}{w (\eta - \gamma)}\right).
\end{align*}
Therefore, it holds
\begin{align*}
    F(x) =\ & \sup_{y}\left\{ xy - (1 - w)(1 - \gamma) F_{0,0}^*\left(y \right) - w(1 - \eta) F_{1,0}^* \left(y \right) - (1 - w) \gamma F_{0,1}^{*} \left(y \right) - w \eta F_{1,1}^{*} \left(y\right)\right\} \\
    =\ & \bigg((1 - w)(1 - \gamma) F_{0,0}^* + w(1 - \eta) F_{1,0}^* + (1 - w) \gamma F_{0,1}^{*} + w \eta F_{1,1}^{*}  \bigg)^{*} (x),
\end{align*}
for all possible $\gamma < w < \eta$. Similar results for $\eta < w < \gamma$ can be obtained by symmetry. 

When $\eta = \gamma = w$, we would like to show
    \begin{align*}
        &T((1 - w) P_0 + w P_1, (1-w) Q_0 + w Q_1) 
        \\
        \geq &(1 - w) T \left(P_0 \| (1 - w) Q_0 + w Q_1 \right) + w T \left(P_1 \| (1 - w) Q_0 + w Q_1 \right).
    \end{align*}
Rewrite $F_1 = T \left(P_0 \| (1 - w) Q_0 + w Q_1 \right)$ and $F_2 = T \left(P_1 \| (1 - w) Q_0 + w Q_1 \right)$.
Lemma \ref{lemma: HS advanced joint convexity} implies that 
    \begin{align*}
        F^*(- e^{\epsilon'}) = \inf_{e^{\epsilon_0}, e^{\epsilon_1}} \left\{ (1 - w) F_1^{*}( - e^{\epsilon_0}) + w F_2^{*} ( - e^{\epsilon_1}) | (1 - w) e^{\epsilon_0} + w e^{\epsilon_1} = e^{\epsilon'} \right\},
    \end{align*}
    where the constraint $(1 - w) e^{\epsilon_0} + w e^{\epsilon_1} = e^{\epsilon'}$ comes from equation \eqref{eqn: joint epsilon} and \eqref{eqn: gamma and beta}. 
    Thus, for any $x \in [0,1]$, we have
    \begin{align*}
        F(x) =\ & \sup_{- e^{\epsilon'}, \epsilon' > 0} - x e^{\epsilon'} - F^{*} (- e^{\epsilon'})\\
        =\ & \sup_{e^{\epsilon'}, \epsilon' > 0} - x e^{\epsilon'} - \inf_{e^{\epsilon_0}, e^{\epsilon_1}} \left\{ (1 - w) F_1^{*}( - e^{\epsilon_0}) + w F_2^{*} ( - e^{\epsilon_1}) | (1 - w) e^{\epsilon_0} + w e^{\epsilon_1} = e^{\epsilon'} \right\}\\
        =\ & \sup_{e^{\epsilon'}} \sup_{e^{\epsilon_0}, e^{\epsilon_1}:(1 - w) e^{\epsilon_0} + w e^{\epsilon_1} = e^{\epsilon'} } - x e^{\epsilon'} - \left\{ (1 - w) F_1^{*}( - e^{\epsilon_0}) + w F_2^{*} ( - e^{\epsilon_1}) \right\}.
    \end{align*}
    According to the properties of infimal convolution in convex analysis (cf., Exercise 12 of Chapter 3.3 in Page 57 of \cite{MR2184742}), we get
    \begin{align*}
         F(x) =\ & \sup_{e^{\epsilon_0}, e^{\epsilon_1}} - (1 - w) x e^{\epsilon_0} - w x e^{\epsilon_1} - (1 - w) F_1^{*}( - e^{\epsilon_0}) - w F_2^{*} ( - e^{\epsilon_1})\\
         =\ & \sup_{e^{\epsilon_0}} - (1 - w) x e^{\epsilon_0} - (1 - w) F_1^{*}( - e^{\epsilon_0}) + \sup_{e^{\epsilon_1}} - w x e^{\epsilon_1} - w F_2^{*} ( - e^{\epsilon_1})\\
         =\ & (1 - w) F_1(x)^{**} + w F_2 (x)^{**} = (1 - w) F_1(x) + w F_2 (x).
    \end{align*}    
    This completes the proof of this corollary. 
\end{proof}

\subsubsection{Proof of Proposition \ref{prop:advancedShuffling}}
\label{sec:proof_prop_advanced_shuffling}

Let $f_0 = \mathcal{C}(T(P_0,Q_0))$ be the symmetrization of $T(P_0,Q_0)$ and rewrite $f_{P,Q} = T(P,Q).$

\begin{proof}[Proof of Proposition \ref{prop:advancedShuffling}]
Since $P = (1-w)P_0 + w Q_0$ and $Q = (1 - w)Q_0 + w P_0$, according to 
Theorem 2 in \cite{balle2018privacy}, we have
\begin{align*}
    H_{e^{\epsilon'}}(P \| Q) &= H_{e^{\epsilon'}}((1-w)P_0 + w Q_0 \| (1 - w) Q_0 + w P_0)\\
    &= H_{e^{\epsilon'}}\left(2w \cdot \frac{ P_0 + Q_0}{2} + (1-2w)P_0 \bigg\| 2w \cdot \frac{ P_0 + Q_0}{2} + (1 - 2w) Q_0 \right)\\
    &\leq (1 - 2 w) H_{e^{\epsilon}} \left(P_0 \bigg\| (1 - \eta) \cdot \frac{ P_0 + Q_0}{2} + \eta Q_0 \right)\\
    &\leq (1 - 2 w) \left(\frac{1}{2} + \frac{\eta}{2} \right) H_{e^{\epsilon}} \left(P_0 \| Q_0 \right)
\end{align*}
with $e^{\epsilon'} = (1- 2 w) e^{\epsilon} + 2 w$ and $\eta = e^{\epsilon'}/e^{\epsilon}$. 
It is clear from the definition that $\eta \leq 1$. 
Therefore, 
\begin{align*}
    H_{e^{\epsilon'}}(P \| Q) \leq (1 - 2 w) H_{e^{\epsilon}} \left(P_0 \| Q_0 \right).
\end{align*}
Let $y = - e^{\epsilon'}$.
Proposition 2.12 in \cite{dong2019gaussian} implies that 
\begin{align*}
    1 + f_{P,Q}^*(y) \leq (1 - 2w) \left(1 + f_{0}^*\left(\frac{y + 2 w}{1 - 2 w} \right) \right)
\end{align*}
Therefore, 
\begin{align*}
    f_{P,Q}(x) =\ & \sup_{y} xy - f_{P,Q}^*(y)\\
    \geq\ & \bigg( - 2w_0 + (1 - 2w_0) f_{0}^*\left(\frac{y + 2 w}{1 - 2 w} \right) \bigg)^*
\end{align*}
By properties of convex conjugate, we have
\begin{align*}
    f_{P,Q}(x) \geq\ & 2w (1-x) + (1 - 2w) f_{0}^{**} (x)\\
    =\ & 2w (1-x) + (1 - 2w) f_{0} (x)
\end{align*}

According to Proposition F.2 in \cite{dong2019gaussian}, the shuffling model is $\mathcal{C}(f_{P,Q})$-DP.
\end{proof}

\section{Tightness of Lemma \ref{lemma:inequality}}
As we see from Proposition \ref{prop:shuffling}, Lemma \ref{lemma:inequality} holds with equality.
However, in general, Lemma \ref{lemma:inequality} is not tight (cf., Figure \ref{fig:compare_shuffling}).
From the technical proof of Proposition \ref{prop:shuffling}, we obtain that $\frac{p_{\bfw}}{q_\bfw} = \sum_{i=1}^m w_i p_i/q_i$, which motivates us to derive Proposition \ref{prop:iff-equality}.

\subsection{Proof of Proposition \ref{prop:iff-equality}}
    By Theorem 2.10 in \cite{dong2019gaussian}, we know that for distributions $P_{\bfw}, Q_{\bfw}$ and ${P_I}, {Q_I}$, it holds
\begin{align*}
    T(P_{\bfw},Q_{\bfw}) \geq T({P_I}, {Q_I}) \quad \hbox{iff} \quad (P_{\bfw},Q_{\bfw}) \succeq_{Blackwell} ({P_I}, {Q_I}).
\end{align*}
We define the Blackwell order as in, for example, \cite{blackwell1953equivalent, dong2019gaussian,raginsky2011shannon}. 
Precisely, if there are probability distributions $P$ and $Q$ on $Y$, as well as probability distributions $P'$ and $Q'$ on $Z$, and a randomized algorithm $\text{Proc}: Y \mapsto Z$ such that $\text{Proc}(P) = P', \text{Proc}(Q) = Q'$, then we write $(P,Q) \succeq_{Blackwell} (P', Q')$.


Let $F_0$ be the cumulative distribution function of the log-likelihood ratio $\log \frac{dP_{\bfw}}{d Q_{\bfw}}(X)$ for $X\sim P_{\bfw}$. $G_0$ is defined analogously by replacing $P_\bfw$ and $Q_\bfw$ with ${P_I}$ and ${Q_I}$, respectively.
Furthermore, we define the perfect log-likelihood function $\Tilde{F}_1(x)$ and $\Tilde{G}_1(x)$ to satisfy the following: 
\begin{align*}
    \Tilde{F}_1(x) = Q_{\bfw} \left(\log \left(\frac{d Q_{\bfw}}{d P_{\bfw}} \right) - E \leq x \right),
\end{align*}
and
\begin{align*}
     \Tilde{G}_1(x) = {Q_I} \left(\log \left(\frac{d {Q_I}}{d {P_I}} \right) - E \leq x \right),
\end{align*}
where $E$ is a random variable such that, under $Q_{\bfw}$, 
$E$ is independent of $\log \frac{d Q_{\bfw}}{d P_{\bfw}}$ and is distributed according to an exponential distribution with support $\mathbb{R}_+$ and cumulative distribution function $1 - e^{-x}$ for all $x \geq 0$. 
By Theorem 3 in \cite{mu2021blackwell}, we know
\begin{align*}
    \Tilde{F}_1(x) \geq \Tilde{G}_1(x),\qquad \text{for all $x \in \mathbb{R}$},
\end{align*}
if and only if
\begin{align*}
    (P_{\bfw},Q_{\bfw}) \succeq_{Blackwell} ({P_I}, {Q_I}).
\end{align*}
Therefore, equality in Lemma \ref{lemma:inequality} holds if and only if $\Tilde{F}_1(x) = \Tilde{G}_1(x)$ for all $x \in \mathbb{R}$.
The following equations \eqref{Tilde{F}_1} and \eqref{Tilde{G}_1} is appear in the proof of Lemma 1 in \cite{mu2021blackwell}. For the sake of thoroughness, we will include a summary of the proof later in this section for reference.
We have
\begin{align} \label{Tilde{F}_1}
    \Tilde{F}_{1}(x) = \int_{-x}^{\infty} F_0(v) e^{-v} dv,
\end{align}
and
\begin{align}
\label{Tilde{G}_1}
    \Tilde{G}_{1}(x) = \int_{-x}^{\infty} G_0(v) e^{-v} dv.
\end{align}

Since $F_0$ and $G_0$ are continuous, equality holds for all $x$ if and only if $F_0(v) e^{-v} = G_0(v) e^{-v}$ by fundamental theorem of calculus. 
We conclude that equality in Lemma \ref{lemma:inequality} holds if and only if 
\begin{align*}
    \frac{w_1 p_1 + w_2 p_2}{w_1 q_1 + w_2 q_2} \overset{\mathbb{P}}{=} w_1 \frac{p_1}{q_1} + w_2 \frac{p_2}{q_2}
\end{align*}
with respect to $P_{\bfw}$.

\begin{proof}[Proof of Equation \eqref{Tilde{F}_1} and \eqref{Tilde{G}_1}]
    We define 
    \begin{align*}
        F_1(v) = Q_{\bfw} \left(\log \frac{dQ_{\bfw}}{d  P_{\bfw}} \leq v \right)
    \end{align*}
    and $\Tilde{F}_1$ to be the convolution of the distribution $F_1$ with the distribution of $-E$, and thus can be written as
    \begin{align*}
        \Tilde{F}_1(x) =\ & \int_{\mathbb{R}} Q_{\bfw}(-E \leq x - u) dF_1(u)\\
        =\ & F_1(x) + e^x \int_{x}^{\infty} e^{-u} dF_1(u)\\
        =\ & \int_{-\infty}^{x} d F_1(u) + e^x \int_{x}^{\infty} e^{-u} dF_1(u).
    \end{align*}
    Moreover, we substitute that $d F_1(u) = -e^u d F_0(-u)$ into equations above. Then, it holds
    \begin{align*}
        \Tilde{F}_1(x) =\ & \int_{-\infty}^{x} - e^u d F_0(-u) + e^x \int_{x}^{\infty}- dF_0(-u)\\
        =\ & \int_{-x}^{\infty} e^{-u} d F_0(u) + e^x \int_{-\infty}^{-x} dF_0(u)\\
        =\ & \int_{-x}^{\infty} - e^u d F_0(u) + e^x F_0(- x).
    \end{align*}
We conclude equation \eqref{Tilde{F}_1} via    integral by part. Equation \eqref{Tilde{G}_1} can be proved similarly. 
\end{proof}

\subsection{Other examples where Lemma \ref{lemma:inequality} holds with equality}
According to Proposition \ref{prop:iff-equality}, we may find other examples in which Lemma \ref{lemma:inequality} holds with equality.
 Straightforward examples are that the support of $w_1 P_1 + w_2P_2$ and the support of $w_1Q_1 + w_2 Q_2$ are disjoint sets.

 \begin{example}
Consider $P_1 = \mathrm{Unif}([0,1])$, $Q_1 = \mathrm{Unif}([2,3])$, $P_2 = \mathrm{Unif}([-1,0])$, and $Q_2 = \mathrm{Unif}([3,4])$. It is easy to verify that the condition in Proposition \ref{prop:iff-equality} holds.
In fact, we have
\begin{align*}
    \frac{P_{\bfw}}{Q_{\bfw}}(X) = \frac{P_1}{Q_1}(X) = \frac{P_2}{Q_2}(X) = \infty
\end{align*}
as the support of $Q_1$ and $Q_2$ are disjoint with $[-1,1]$.
 \end{example}

\begin{example}
Another example where the equality holds in Lemma \ref{lemma:inequality} is that $P_{1}$ and $Q_{1}$ are two probability distributions supported on the $x$-axis, and $P_{2}$ and $Q_{2}$ are two probability distributions supported on the $y$-axis. 
If the point mass at $0$ are all $0$, then one can show that $P_i$ and $Q_i$ for $i \leq 2$ satisfy the equality condition in Lemma \ref{lemma:inequality}. 
To see this, let $X = (X_1, X_2) \in \mathbb{R}^2$ be a random variable with distribution $w_1 P_{1} + w_2 P_{2}$. Therefore, $X$ is supported on the axes. 
For any $t \geq 0$, the right hand side of Proposition \ref{prop:iff-equality} becomes
    \begin{align*}
        &\mathbb{P} \left(\frac{w_1 p_1 + w_2 p_2}{w_1 q_1 + w_2 q_2} (X) \leq t \right)\\ 
        =\ & w_1 \mathbb{P} \left(\frac{w_1 p_1 + w_2 p_2}{w_1 q_1 + w_2 q_2} (X) \leq t \bigg|  X \sim P_1 \right) + w_2 \mathbb{P} \left(\frac{w_1 p_1 + w_2 p_2}{w_1 q_1 + w_2 q_2} (X) \leq t \bigg|  X \sim P_2 \right) \\ 
        =\ & w_1 \mathbb{P} \left(\frac{p_1}{q_1} (X) \leq t \bigg|  X \sim P_1 \right) + w_2 \mathbb{P} \left(\frac{p_2}{q_2} (X) \leq t \bigg|  X \sim P_2 \right),
    \end{align*}
    which is exactly the right-hand side of Proposition \ref{prop:iff-equality}. 
\end{example}

\subsection{Comparisons between Lemma \ref{lemma:inequality} and Lemma \ref{lemma: f advanced joint convexity}}
In this section, we compare Lemma \ref{lemma:inequality} with Lemma \ref{lemma: f advanced joint convexity} using other examples besides that given by  Figure \ref{fig:Tightness}.
The first example is a simple case appears in sub-sampled Gaussian mechanisms.

\begin{figure}
    \centering
    \includegraphics[scale = 0.4]{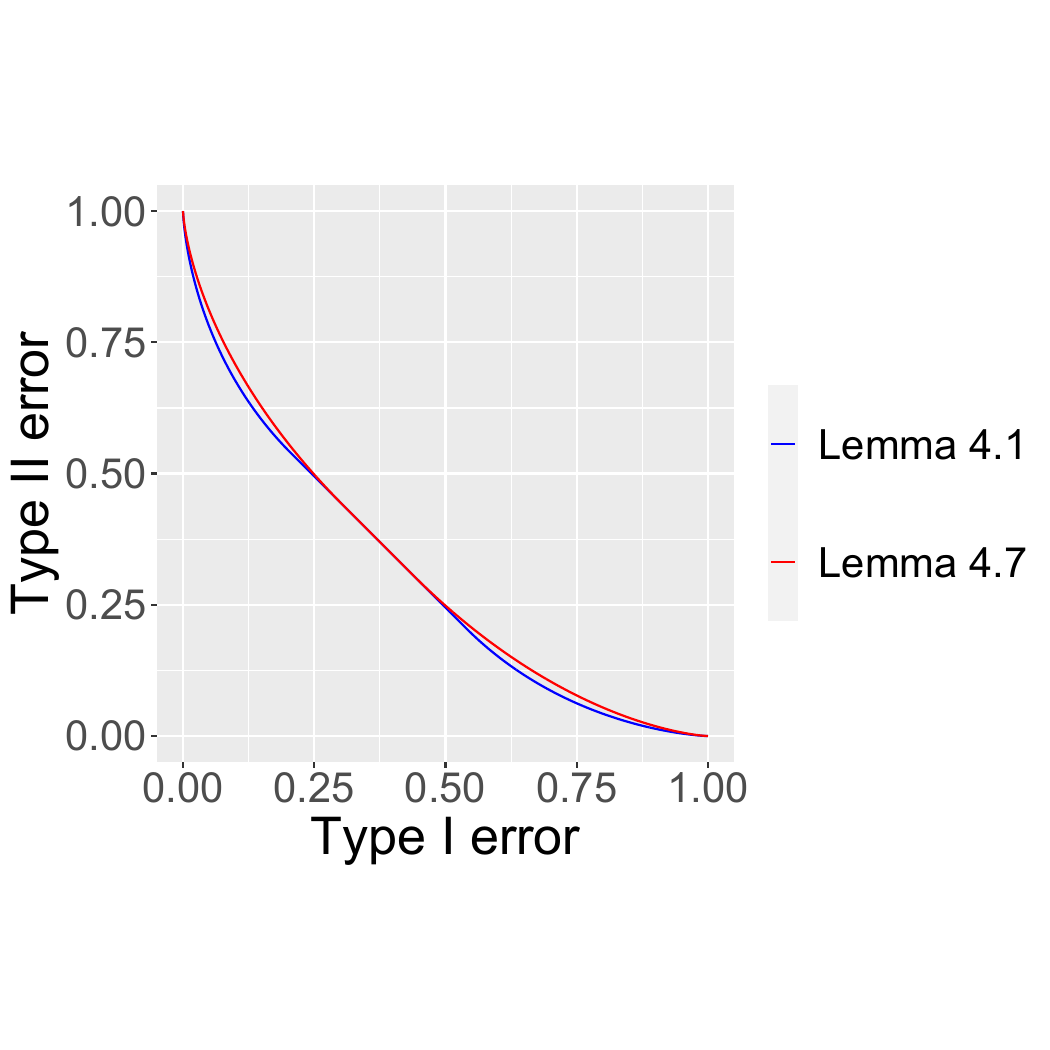}
    \caption{Example \ref{exam:sub-sample} with $\mu=1$ and $w = 1/3$.}
    \label{fig:sampling}
\end{figure}
\begin{example}
\label{exam:sub-sample}
    Let $P_0 = \mathcal{N}(0,1)$ and $Q_0 = \mathcal{N}(\mu,1)$.    
    Then, using Lemma \ref{lemma: f advanced joint convexity}, we have $T(P_0,w P_0 + (1 - w) Q_0) \geq \mathcal{C}(w \cdot \mathrm{Id} + (1 - w) T(P_0,Q_0)).$
    Lemma \ref{lemma:inequality} leads  to the following lower bound 
    \begin{align*}
    T(P_0,w P_0 + (1 - w) Q_0)(\alpha) \geq \left\{
     \begin{array}{lll}
       (1-w)\Phi( \Phi^{-1}\left(\frac{1 - \alpha}{1-w} \right) - \mu),   &  \alpha > 1 - (1 -w) \Phi(\frac{\mu}{2}), \\
        w + (1- w) \Phi\left(\Phi^{-1}\left(1 - \frac{\alpha}{1 - w}\right) - \mu\right) , & \alpha < (1 - w) (1 - \Phi(\frac{\mu}{2})),  \\
         \hbox{linear}, & \hbox{otherwise}.
     \end{array}
    \right.
    \end{align*}
As we see from Figure \ref{fig:sampling}, Lemma \ref{lemma: f advanced joint convexity} leads to a slightly tighter bound.
\end{example}

Another example is an extreme case where the mixture distributions are not distinguishable at all.

\begin{example}
Consider the case $(\frac{1}{2}P + \frac{1}{2} Q) \hbox {v.s.} (\frac{1}{2}Q + \frac{1}{2} P)$ where two distributions are not distinguishable at all. We have $T\left((\frac{1}{2}P + \frac{1}{2} Q), (\frac{1}{2}Q + \frac{1}{2} P) \right) = \mathrm{Id}$ which is a special case of Proposition \ref{prop:advancedShuffling}.
The advanced joint concavity Proposition \ref{prop:advancedShuffling} leads to a sharp lower bound $\mathrm{Id}$.
However, it is obvious that Lemma \ref{lemma:inequality} implies a loose bound when $P\neq Q$.
\end{example}

\end{document}